\documentclass[11pt]{article}
\usepackage[margin=1in]{geometry}
\usepackage{graphicx}
\usepackage{amsmath, amssymb}
\usepackage{amsthm}  
\allowdisplaybreaks
\usepackage{authblk}
\usepackage{cite}
\usepackage{hyperref}
\usepackage{caption}
\usepackage{subcaption}
\usepackage{float}
\usepackage{xcolor}
\usepackage{algorithm}
\usepackage{algpseudocode}
\usepackage{orcidlink}

\usepackage{tikz}
\usepackage{pgfplots}
\pgfplotsset{compat=newest}

\newtheorem{theorem}{Theorem}[section]     
\newtheorem{proposition}[theorem]{Proposition} 

\definecolor{darkgreen}{RGB}{0,128,0} 

\title{Information Processing Capacity of Stationary Physical Systems: Theory,
Data-efficient Estimation Methods, and Photonic Demonstration}

\author[1]{Rahul Uma Ramachandran,\orcidlink{0009-0009-1955-0671}}
\author[1]{Serge Massar,\orcidlink{0000-0002-4381-2485}}
\affil[1]{Laboratoire d'Information Quantique CP224, Université libre de Bruxelles, Av. F. D. Roosevelt 50, B-1050 Bruxelles, Belgium}

\date{}

\begin{document}

\maketitle

\begin{abstract}

Physical computing systems provide a promising route toward hardware-native machine learning, but their computational capabilities remain difficult to characterize in a principled, task-independent, and data-efficient way. We extend the Information Processing Capacity (IPC) framework to stationary physical computing systems and establish several fundamental results: individual capacities are bounded between zero and one, their sum over a complete basis is bounded by the number of readouts, and noise strictly reduces this bound. We address the finite-sample estimation of IPC and derive the asymptotic form of the systematic positive bias affecting naive estimators.
Building on these results, we introduce data-efficient estimation methods based on Richardson extrapolation and  Sobol quasi-random sampling. We validate the framework experimentally using a photonic computing system based on picosecond laser pulses propagating through a nonlinear optical fibre. By varying the laser power and fibre length, we observe systematic shifts of the IPC distribution toward higher-order nonlinear capacities induced by the Kerr effect.
Finally, we demonstrate that the total IPC strongly correlates with performance on benchmark machine-learning tasks and provides a reliable estimate of the effective dimensionality of the system. These results establish IPC as a practical bridge between the intrinsic dynamics of physical computing systems and their machine-learning performance.

\end{abstract}

\section{Introduction}
 The slowdown of Moore’s law, increasing energy costs, and the growing demand for massively parallel computation driven by machine learning are pushing conventional general-purpose computing systems toward their limits. This has led to a search for alternative computing paradigms that are less affected by these factors. Physical computing is one such promising paradigm. A vast array of physical systems around us naturally performs complex information processing, and physical computing architectures aim to harness these to speed up computations without sacrificing energy efficiency \cite{markovic2020physics,huang2015trends,jaeger2023toward}. During the past years, machine learning has been realized using many different physical substrates and architectures, see e.g.\cite{tanaka2019recent,markovic2020physics,schuman2022opportunities}
 
 However, the exact nature of the computations and nonlinearities performed by most physical systems is complex and therefore difficult to characterize and predict by the experimentalist. As a result, tools to properly measure and classify the computational traits of a physical system are invaluable. Several metrics have been used to benchmark these computational traits, see e.g. \cite{nonlinear_inference_2spiral,edgeofchaos,neurobench,blouw2019benchmarking,principles_and_metrics_elm}. However, many of these do not fully satisfy desirable properties one would expect from such a metric, such as task-independence,  learning-algorithm independence, hardware-platform independence, and data-efficiency. Approaches which use Principal Component Analysis (PCA)\cite{hotelling1933pca,malinowski} to estimate the dimensionality of a system satisfy most of these properties, but provide only limited information about the actual computations performed by the system.  
 
 The Information Processing Capacity (IPC), introduced for linear functions in  \cite{jaeger2002short}  and extended to nonlinear functions in \cite{dambre2012information}  to analyze dynamical systems such as reservoir computers, is a strong candidate for such a metric. The theory of IPC was further developed in \cite{gonon2020memory,infocap_timevar,SAITO2026132128}. IPC has been widely used in both numerical studies, see e.g. \cite{short_ortho,rodan2010minimum,grigoryeva2015optimal,ortin2020delay,infocap_spintronic_bifurcation,mastermemory,infocap_task_specific,koster2023limitations,infocap_quantumspin,jaurigue2025utilizing,gallicchio2018short,nokkala2021gaussian,martinez2023information}, and experimental studies, see e.g. \cite{duport2012all,infocap_singlenode,harkhoe2019task,tsunegi2023information,shen2023deep}, of reservoir computers. Limitations of the  Linear IPC were addressed,  particularly concerning numerical precision in linear networks, in \cite{JMLR:v25:23-0568}, and its sensitivity to input scaling in the case of nonlinear networks in \cite{Ballarin26}. However, the study of the IPC in the context of reservoir computing is complicated by its inherent temporal nature, as the current state of a reservoir depends on its entire input history. This leads to technical difficulties, which sometimes mask the simplicity of the method.
 
Here, we develop the theory and applications of the Information Processing Capacity of stationary physical systems. Stationary systems have wide applicability in machine learning, for instance, in feedforward networks and in Extreme Learning Machines (ELM) \cite{HUANG2006489} (which can be thought of as a single-layer feedforward network). We refer to  \cite{shen2017deep,yao2020fully,wright2022deep,ross2023multilayer} for some experimental implementations of feedforward networks and to \cite{frances2016hardware,pierangeli2021photonic,suprano2024experimental,Zajnulina:25} for some experimental implementations of ELMs.
The IPC of stationary systems has been studied previously in \cite{hu2023tackling}, but with a focus on the specific aspects introduced by sampling noise, such as encountered in quantum systems. 

Our first contribution is to adapt the theory of IPC introduced in \cite{dambre2012information} to stationary systems. This is quite straightforward, since stationary systems can be viewed as a particular case of the general time-dependent framework. We reprove two key results, namely (i) that the capacity $C(y,X)$ to reconstruct any function  is bounded by $0\leq C(y,X) \leq 1$,
where $y$ is the target function and $X=\{X_1,\ldots,X_K\}$ are the $K$ readouts
; and (ii) that the sum of capacities over a complete basis is bounded by the number of readouts $K$.
The simpler theoretical framework of stationary systems 
 enables us to focus on questions that were overlooked in the context of reservoir computing. 

 Our second contribution concerns the effect of noise on the IPC. This question was investigated for linear dynamical systems in \cite{short_ortho} but has not received much further theoretical attention, although it is ubiquitous in experimental systems.
 We show that in the presence of noise, the sum of the capacities over a complete basis is always strictly less than the number of readouts. This result supports the interpretation of the IPC as measuring the total accessible information in a physical system.
 
Our third contribution concerns how to estimate the capacities in practice when a limited number $N$ of data points are available. To this end, we address the systematic positive bias that affects the IPC when the amount of data is finite. This positive bias is particularly problematic when computing the total IPC by summing the capacities over a complete basis, since one can violate -sometimes substantially- the upper bound given by the number of readouts.
A preliminary discussion of this difficulty appeared in the original work \cite{dambre2012information}, and it was revisited in \cite{JMLR:v25:23-0568,SAITO2026132128}. However, \cite{JMLR:v25:23-0568} only addresses the Linear IPC, while the method to correct for the positive bias proposed in \cite{SAITO2026132128} is not data-efficient. 

We show, see also
\cite{SAITO2026132128}, that 
$C_N(y,X) = C(y,X) +  
\Delta(y,X)/{N} + {\mathcal O}\left({N^{-3/2}}\right)$  where $C_N(y,X)$ is the IPC estimated using a finite number $N$ of samples, $\Delta(y,X)\geq 0$ is a positive function of the target $y$ and the readouts $X$. The positivity of $\Delta(y,X)$ is the systematic bias mentioned above. We derive an analytical $y$-dependent upper-bound for $\Delta(y,X)$.
We use this asymptotic form and upper-bound to propose data-efficient estimation techniques that correct this systematic bias while keeping the estimated capacities bounded between $0$ and $1$. 
In addition, we show empirically that the use of Sobol sampling \cite{SOBOL196786}  reduces the errors made when estimating capacities on a finite number $N$ of samples, compared to the use of pseudo-random sampling.

 Turning to practical use, we introduce two useful methods for visualizing capacities. These representations aim to encode the computational traits of a physical system in a single picture. These plots can be used as standard overviews of the potential of physical computing systems to process information, and readily compared with each other.

Our fourth contribution consists of illustrating the above results on a concrete experimental system based on a picosecond pulsed laser propagating through a nonlinear medium. The inputs are encoded by modulating the spectrum of the laser. The pulses are then amplified before propagating through an optical fibre. Finally, the output spectrum is measured, and the power in distinct spectral bins constitutes the output variables $X$. 
 Similar experimental systems have already been used as Extreme Learning Machines in \cite{Lupo:21,Zajnulina:25,principles_and_metrics_elm,programmable_photonic_elm,fission}. 
 We use the above mentioned data-efficient estimation method to evaluate the IPCs.
  We use the two visualization methods introduced earlier to show how the IPCs depend on the controllable experimental parameters (power of the laser and length of the fibre).
  We show that the total IPC correlates with performance on two machine learning tasks, namely PCA reduced MNIST \cite{mnist} and Two Spirals \cite{LangWitbrock1988}. 

Finally, we discuss how the IPCs can be used to estimate the effective dimensionality of the system.
  Recall that the sum of the capacities over a complete basis is upper-bounded by the number of readouts. However, in practice, the sum is strictly less than the upper bound, due to inevitable noise (as mentioned above), finite statistics, and the fact that the infinite basis is truncated to a finite number of terms. It is tempting to nevertheless interpret the sum of the IPCs as an estimate of the effective dimensionality of the system. To validate this interpretation, we compare the total IPC with another metric used to estimate the effective dimensionality of stationary physical systems  based on Principal Component Analysis
 \cite{malinowski}.  For a recent use of this approach in a photonic system, see
 \cite{Skalli:22}. We find good agreement between the two approaches.

 In summary, in the present work, we extend the concept of Information Processing Capacity to stationary systems, obtaining new theoretical results and improved methods to use the IPC in practice. These concepts are then illustrated on an experimental photonic system. The results and methods developed here can also be used, possibly with small adaptations, in the case of dynamical systems and reservoir computers, implying broad applicability beyond the stationary systems studied here.

\section{Principles}

In this section, we adapt the notions introduced in \cite{dambre2012information} to the case of stationary systems.

\subsection{Model of a stationary  system}

We model a stationary system as a black box which can be interacted with through a set of $q$-dimensional inputs $u = (u_1, \ldots, u_q) \in U \subseteq \mathbb{R}^q$. 
Given input $u$, the system settles into a stationary state determined by these inputs. Its response is probed through a set of $K$ readouts $X=\{X_1,X_2,...X_K\}$ which are functions of the inputs:
\begin{align}
    X: \mathbb{R}^q \to \mathbb{R}^K: u \to  X(u).
\end{align}
We are interested in how this system processes the input information through the responses $X(u)$. The average over the inputs is denoted either by $\langle \cdot \rangle$ or $E[\cdot]$ depending on which is most natural.

\subsection{Extreme Learning Machines (ELMs)}

ELMs are machine learning systems closely related to feedforward neural networks, but with a key difference: the hidden layer is fixed and not trained using methods such as backpropagation. Only the output layer is trained, typically by solving a linear regression problem. 
The output  $\hat{y}$ of the ELM is thus obtained by taking linear combinations of the readouts of the stationary  system:
\begin{equation}
    \hat{y}(u) = \sum_{i=1}^{K} W_{i}\, X_{i}(u) .
\end{equation}

Given a target function $y(u)$, the weights $W_i$ are optimized by minimizing the Mean Squared Error (MSE)
\begin{equation}
    \text{MSE}=\langle ( \hat{y} - y)^2 \rangle
    \label{Eq:DefMSE}
\end{equation}
where the average is taken over the inputs $u$.

After optimization, the optimized weights $W^*$, 
the optimized output $y^*$, and
the optimized Mean Squared Error $\text{MSE}^*$ can be written as:
\begin{eqnarray}
  W^*&=&G^{-1}R ,  \label{Eq:DefW*} \\
  y^* &=& W^{*T}X = R^T G^{-1} X ,  \label{Eq:Defy*}\\
        \text{MSE}^*&=&\langle y^2 \rangle -R^T G^{-1}R ,  \label{Eq:DefMSE*}
\end{eqnarray}
where
the correlation vector $R$ is given by
\begin{align}
   R_{i}&=\langle X_{i} y \rangle  \label{Eq:DefR}.
\end{align}
and
the Gram matrix $G$  is given by
\begin{align}
    G_{ij}&=\langle X_i X_j \rangle .  \label{Eq:DefG}
\end{align}

In Eqs. (\ref{Eq:DefW*}, \ref{Eq:Defy*}, \ref{Eq:DefMSE*}) we assumed that the Gram matrix $G$ is full rank. If it is not, we restrict ourselves to its support, i.e., we take the Moore-Penrose pseudo inverse. Except where explicitly stated, we always make this hypothesis.

\subsection{Probability distribution over inputs and Hilbert space}

We now consider an abstract setting in which the inputs $u \in U$ and the target output $y$ are not associated with any real-world task. The inputs $u$ are independently and identically drawn (i.i.d.) from a probability distribution $p(u)$.

 Given a function over the input space, $f: U \rightarrow \mathbb{R}$, its expectation value is 
 \begin{equation}
   E[f]=\int du \; p(u) f(u)
\end{equation}\\
We  consider the Hilbert space $\mathcal{H}$ of square integrable functions, i.e. functions $f: U \rightarrow \mathbb{R}$ such that
\begin{equation}
E[ f^{2}]=\int du\; p(u) f^{2}(u)<\infty
\end{equation}
 For any functions $f, g \in \mathcal{H}$, we have the scalar product:
\begin{equation}
    \langle f, g\rangle=E[fg] = \int du\; p(u) f(u) g(u)
\end{equation}
and the norm square: 
\begin{equation}
    \|f\|^2 = \langle f,f \rangle = E[f^2] .
\end{equation}

We assume the Hilbert space $\mathcal{H}$  is separable, i.e., has a countable basis. Let  $\{ y_l\}$ be an orthonormal basis of $\mathcal{H}$. We have
\begin{align}
    \left\langle y_{l}, y_{l^{\prime}}\right\rangle=\delta_{l l^{\prime}} & \quad \text { (orthonormality) } \\
    \forall f \in \mathcal{H}, f(u)=\sum_{l} c_{l} y_{l}(u), \text { with } c_{l}=\left\langle y_{l}, f\right\rangle & \quad \text { (completeness) }.
    \end{align}
 The scalar product of two functions can be expressed in terms of their coefficients in the orthonormal basis:
 \begin{align}
 \label{eq:scalar_product}  
& \forall f, g \in \mathcal{H} \text { with } f(u)=\sum_{l} c_{l} y_{l}(u), g(u)=\sum_{l} d_{l} y_{l}(u) , \nonumber  \\
 & \langle f, g\rangle=\sum_{l} c_{l} d_{l}
\end{align}

Hereafter, we consider that all functions belong to $\mathcal{H}$. We use either the expectation notation or the scalar product notation, depending on which is more natural in the specific context.

\subsection{Information Processing Capacity of a stationary system}

\subsubsection{ Capacity to reconstruct a target function}

Given a  target function $y(u)$, the capacity  of the stationary system to reconstruct $y$  from its readouts $X$ is defined as:
\begin{equation}
C(y, X) = 1 - \frac{\text{MSE}^*}{\langle y^2 \rangle}.
\label{Eq:Cdef}
\end{equation}
Using \eqref{Eq:DefMSE*}, we have the explicit expression
\begin{equation}
C(y, X) = \frac{R^\top G^{-1} R}{\langle y^2 \rangle} .
\label{Eq:Cexplicit}
\end{equation}

\begin{proposition} 
\label{Prop:CapBound}
The capacity of a stationary system to reconstruct a target function $y$ is bounded by:
\begin{equation}
0 \leq C(y, X) \leq 1
\end{equation}
\end{proposition}

\begin{proof} The  $MSE$ is positive, see Eq. \eqref{Eq:DefMSE}. Therefore, from the definition Eq.  \eqref{Eq:Cdef}, we have $C(y, X) \leq 1$. 
The Gram matrix $G$ is symmetric and positive (we assume it is full rank). Hence, using the expression Eq. \eqref{Eq:Cexplicit}, we have $0 \leq C(y, X)$.
\end{proof}

Let
\begin{equation}
\mathcal{X} = \mathrm{span}\{X_1,\dots,X_K\}\subset \mathcal H
\end{equation}
denote the subspace spanned by the readout functions and
denote by 
\begin{equation}
P_{\mathcal X}: \mathcal{H} \to \mathcal{H}: y \to  P_{\mathcal X} (y) \end{equation}
the orthogonal projector onto $\mathcal X$.

\begin{proposition} 
The optimized output $y^*$ is the orthogonal projection of $y$ onto $\mathcal X$:
\begin{equation}
y^* = P_{\mathcal X} (y) .
\end{equation}
 The capacity for a stationary system to reconstruct a function $y$ is the square cosine of the angle between the vector $y$ and its projection on the subspace:
\begin{equation}
 C(y, X) = \frac{   \| P_{\mathcal X}(y)   \|^2 }{  \| y   \|^2}.
\end{equation}
\end{proposition}

\begin{proof}
 Assuming the Gram matrix $G$ is full rank, its eigenvalues are all strictly positive. 
 Therefore, there exists an invertible matrix $\Lambda$ such that:
\begin{equation}
\Lambda^\top G \Lambda = \mathbb{I}
\end{equation}
(i.e., $\sum_{ji} \Lambda_{kj} G_{ji} \Lambda_{il} =\delta_{kl})$. We define new variables $\tilde{X} = \{\tilde{X}_l\}$ by:
\begin{equation}
\tilde{X}_l = \sum_i X_i \Lambda_{il}
\label{Eq:tildeX}
\end{equation}
 These satisfy:
\begin{align}
\label{eq:ortho}
\tilde{G}_{ij}&=\langle \tilde{X}_i, \tilde{X}_j \rangle = \delta_{ij} \\
\tilde{R}_{i}&=\langle\tilde{X_i},y \rangle
\end{align}
Thus, the $\tilde{X}_l$ form an orthonormal set in $\mathcal{H}$. 
The orthogonal projector $P_{\mathcal X}$ onto $\mathcal X$ can be expressed as
\begin{equation}
P_{\mathcal X}:  y \to  P_{\mathcal X} (y) = \sum_i \tilde{X}_i \langle \tilde{X}_i, y \rangle 
\end{equation}
and the complementary projector is
\begin{equation}
P^\perp_{\mathcal X} = \mathbb{I} - P_{\mathcal X} .
\end{equation}
Note that using the new basis $\tilde X$ rather than the old basis $X$  leaves the capacity invariant
\begin{equation}
 C(y, X) = C(y, \tilde{X})
\label{Eq:CCequal}
\end{equation}
since any invertible linear transformation of the readouts $X_i$ can be absorbed in the coefficients $W_i$.

The optimized output thus becomes:
\begin{equation}
y^* = \sum_i \tilde{W}_i \tilde{X}_i
\end{equation}
with optimal weights:
\begin{equation}
\tilde{W}_i^{\text{*}} =\tilde{G}_{ij}^{-1}\tilde{R}_{i}= \langle \tilde{X}_i, y \rangle .
\end{equation}
Therefore, the optimized output is the orthogonal projection of $y$ onto $\mathcal X$:
\begin{equation}
y^* = P_{\mathcal X}(y) .
\end{equation}
Furthermore, using the fact that $P_{\mathcal X}^2=P_{\mathcal X}$, we have:
\begin{eqnarray}
C(y, \tilde{X}) &=& \frac{ \sum_i  \langle \tilde{X}_i, y \rangle^2}{\langle y^2\rangle}\label{Eq:CtildeX}\\
&=&  \frac{ \langle y, P_{\mathcal X}(y)\rangle }{\langle y^2\rangle}
\nonumber\\
&=&  \frac{ \langle  P_{\mathcal X}(y) , P_{\mathcal X}(y)\rangle }{\langle y^2\rangle}
\nonumber\\
&=&  \frac{  \|  P_{\mathcal X}(y)  \| ^2 }{ \|  y  \|^2} .
\end{eqnarray}
\end{proof}

\subsubsection{Upper bound on sums of capacities}
\label{sec:capsum_bound}

The sum of the capacities over a complete orthonormal basis is bounded by the number of readouts $K$.

\begin{proposition}  
\label{Prop:SumCap}
Let $X = \{X_i \in \mathcal{H}, i = 1, ..., K\}$ be the set of $K$ readouts of a stationary system, $\mathcal X$ the space spanned by the $X_i$, and let $ \{y_l\}$ be a complete orthonormal basis of $\mathcal{H}$. Then  we have
\begin{equation}
\sum_{l} C(y_l, X) =\dim(\mathcal X) \leq K\ ,
\end{equation}
with equality if and only if the  $K$ readouts are linearly independent.
\end{proposition}

\begin{proof} Let $\tilde{X}$ be the orthonormal functions constructed from $X$ as in \eqref{Eq:tildeX}, thus obeying \eqref{eq:ortho}. 
Let $K' \leq K$ be the dimension of $\mathcal X$. Then there exists an orthonormal basis
\(
\{\tilde X_i\}_{i=1}^{K'}
\)
of \(\mathcal X\).
We have
\begin{align}
\sum_{l} C(y_l, X) &= \sum_{l}  C(y_l, \tilde{X}) \notag \\
&= \sum_{l}  \sum_{i=1}^{K^{'}} \langle \tilde{X}_i, y_l \rangle^2 \notag \\
&= \sum_{i=1}^{K'} \left( \sum_{l}  \langle \tilde{X}_i, y_l \rangle^2 \right) \notag \\
&= \sum_{i=1}^{K'}  \|\tilde X_i\|^2  \notag \\
&= K'
\end{align}
where we have used Eq.~\eqref{eq:scalar_product}.
\end{proof}

In practice, one cannot measure all the capacities $C(y_l, X)$ since there are an infinite number of them. However, one can select a subset, say the $L$ first basis functions, and estimate $\sum_{l=1}^L C(y_l, X)$. This is an important quantity, as it can be interpreted as the effective dimensionality of the stationary system.

\subsubsection{Example based on Legendre polynomials}
\label{subsubsec:examp_legendre}
As an example, consider first the case where we have only a single input (i.e., the input dimension is $q = 1$). We take $p(u)$ to be the uniform distribution over the interval $[-1, +1]$. We then take as an orthonormal basis the Legendre polynomials $\mathcal{P}_{l}(u)$
normalized such that:
\begin{align}
    \frac{1}{2}\int_{-1}^{1} \mathcal{P}_i(u)  \mathcal{P}_j(u)  du= \delta_{ij}
\end{align}

 For multidimensional inputs $(q > 1)$ we take each component of the input
 $u = (u_1, \ldots u_q)$ to be independently and identically distributed (i.i.d.), that is, u is uniformly distributed over the hypercube $[ -1, +1]^q$. We then take the basis to be the products of these Legendre polynomials:
\begin{equation}
    y_{l_1 l_2\cdots l_q } (u)= \mathcal{P}_{l_1}(u_1) \mathcal{P}_{l_2}(u_2) \cdots \mathcal{P}_{l_q}(u_q)
\end{equation}
normalized such that:
\begin{equation}
  \langle y_{l_1 l_2\cdots l_q } , y_{l'_1 l'_2\cdots l'_q }  \rangle=
\delta_{l_1 l'_1}\cdots   \delta_{l_q l'_q}  
\end{equation}
This is the basis used throughout the paper.

For q-dimensional inputs, if we restrict to basis functions satisfying
\[
l_1+\cdots+l_q\le d_{\max}
\], i.e., to having a total degree less than or equal to $d_{max}$, then the number of basis functions is
\begin{equation}
    n_{basis}=\binom{q+d_{max}}{d_{max}} .
\end{equation}

In addition to Legendre polynomials, one can also use any other set of orthonormal functions together with the associated probability distribution, see   \cite{infocap_timevar} for a discussion.

\section{Visualization of capacities}

Before discussing the capacities further,  we introduce two visualization schemes that allow the capacities to be visualised and interpreted. Examples of the proposed visualizations are provided in Fig. \ref{Fig:IllustrativeCapacities}.

\begin{figure}[H]
    \centering
    \includegraphics[width=0.9\textwidth]{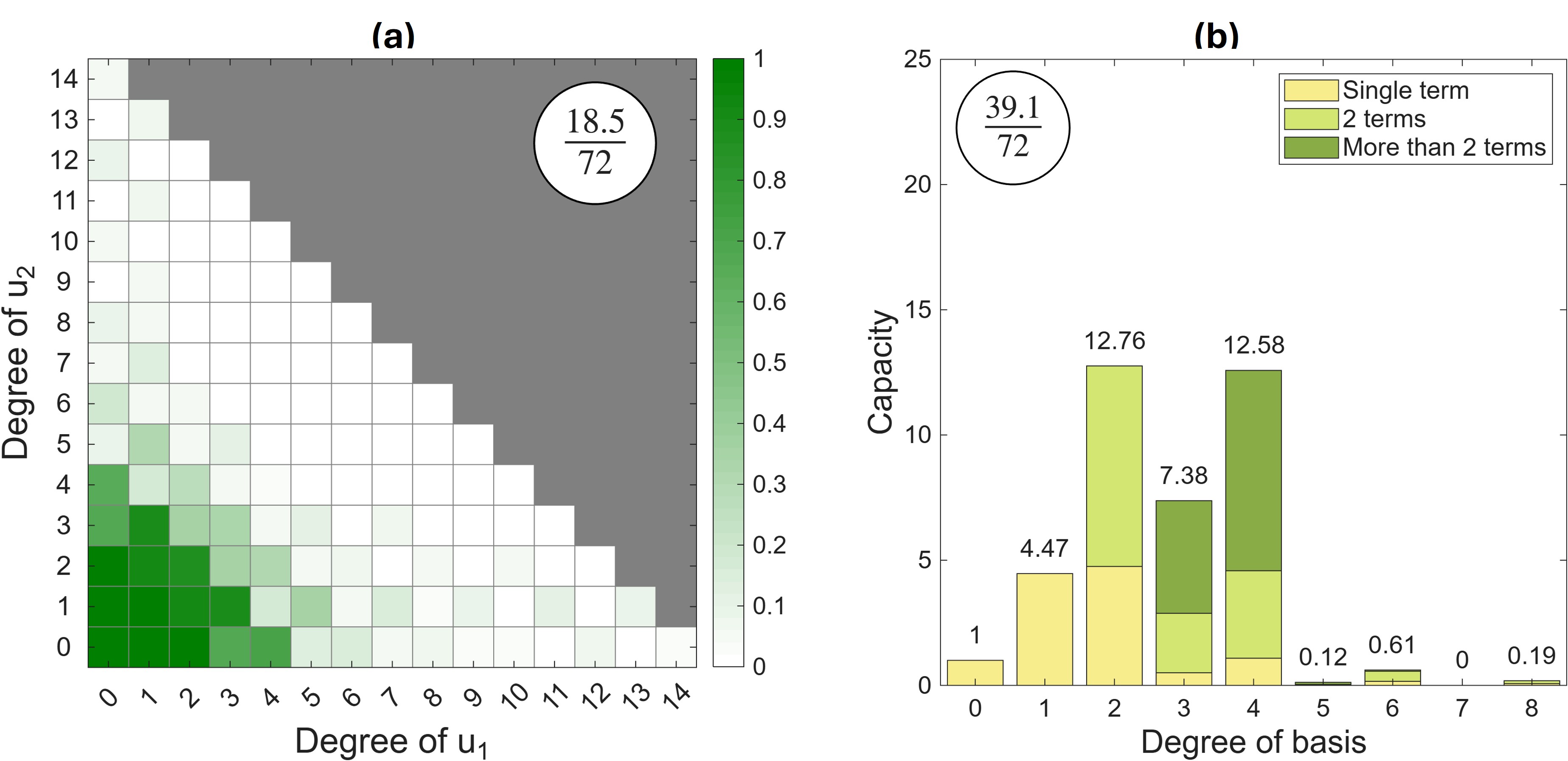}
    \caption{Illustration of capacity visualizations: (a) Capacity matrix when the input dimensionality is $q=2$; (b) Capacity Bar-plot used for high-dimensional input, here the dimensionality is $q=5$. We use the Legendre polynomials and their products as an orthonormal basis, see Section \ref{subsubsec:examp_legendre}. The data is obtained using the experimental system described in Section \ref{Sec:Exp-Sys} using average power=-1.9 dBm and fiber length =40m.}
    \label{Fig:IllustrativeCapacities}
\end{figure}

\begin{enumerate}
    \item \textbf{Capacity matrix}\\
    When the inputs are 2-dimensional $u=\{u_1,u_2 \}$, the basis used consists  of products of two Legendre polynomials $\mathcal{P}_{l_1}(u_1) \mathcal{P}_{l_2}(u_2)$. The corresponding capacities can therefore be visualized as a matrix, where the row and column correspond to degrees of $l_1$ and $l_2$ of each polynomial. Each matrix element is shown as a color gradient according to the capacity of that specific term. For example, if the indexing starts at 0, the $(0,4)$ element corresponds to the basis $\mathcal{P}_0(u_1) \cdot \mathcal{P}_4(u_2)=P_4(u_2)$ while  $(2,3)$ corresponds to $\mathcal{P}_2(u_1) \cdot \mathcal{P}_3(u_2)$. The (0,0) entry corresponds to the constant term $\mathcal{P}_0=1$. The capacity matrix in our case is always lower triangular because we set a limit on the maximum degree of the basis (here, 14), and all terms above the diagonal correspond to higher degrees. We indicate these invalid cells in gray.
    \item \textbf{Capacity bar-plot}\\
    For higher-dimensional inputs, capacities corresponding to bases with the same total degree can be grouped together to visualize how capacity is distributed across different degrees. To do this, we construct a bar plot where the x-axis represents the total degree of the basis polynomial and the y-axis shows the sum of capacities of all terms with that total degree.
    
    For example, for a three-dimensional input $u = \{u_1, u_2, u_3\}$, the bases 
    $\mathcal{P}_1(u_1)\mathcal{P}_2(u_2)\mathcal{P}_3(u_3)$ and 
    $\mathcal{P}_2(u_1)\mathcal{P}_2(u_2)\mathcal{P}_2(u_3)$ both have total degree $6$ and therefore contribute to the same bar.
    
    Each bar is further divided into three categories based on the number of input variables involved:
    \begin{enumerate}
        \item Single-term contributions (e.g., $\mathcal{P}_6(u_1)$),
        \item Two-term interactions (e.g., $\mathcal{P}_4(u_1)\mathcal{P}_2(u_2)$),
        \item Higher-order interactions involving more than two terms (e.g., $\mathcal{P}_1(u_1)\mathcal{P}_2(u_2)\mathcal{P}_3(u_3)$).
    \end{enumerate}
    This highlights how the computational properties of the system are distributed across nonlinear orders. 

\end{enumerate}
In all cases, the Total Capacity is represented as a fraction: $\frac{\textbf{Total Capacity}}{\textbf{Maximum Capacity}}$ inside a circle, where {\bf Total Capacity} is the sum of all the measured capacities, and {\bf Maximum Capacity} equals the number $K$ of readouts.

\section{Effect of noise on the capacities}

\subsection{General case}

 Noise is inevitable in experimental systems. By noise, we mean uncontrolled fluctuations that are statistically independent of the inputs and vary across repeated measurements at a fixed input. Here, we show that the presence of noise decreases the capacities. We first prove a general result which is independent of the noise model, and then specialize to the case of additive noise. The effect of noise on the linear memory capacity of time-dependent systems was studied previously in \cite{short_ortho}.

In the presence of noise, the readouts depend on both the input $u$ and a noise variable $\epsilon\in \mathbb{R}^M$ which, for simplicity, is assumed to belong to a finite-dimensional space of dimension $M$:
\begin{equation}
X: \mathbb{R}^q\times\mathbb{R}^M \to \mathbb{R}^K: (u,\epsilon) \to  X(u,\epsilon) .
\end{equation}

The input variables $u$ and the noise variables $\epsilon$ are, by assumption, independent. 

It is mathematically convenient to treat the noise as an additional input variable, even though its realizations are neither accessible nor controllable.
Mathematically, this puts the noise on the same footing as the input $u$. But since we do not know the value of the noise variable $\epsilon$, we need to average over it at the end of all computations.

We therefore suppose that there exists a probability distribution over the noise $p(\epsilon)$ which defines the corresponding noise Hilbert space $\mathcal{H}_\epsilon$, and an orthonormal basis over the noise space which we denote $z_m(\epsilon)$,  $m\in \mathbb{N}$. The readouts now belong to the tensor product of the Hilbert spaces $X_i\in \mathcal{H}\otimes\mathcal{H}_\epsilon $. Independence means that the probability distribution over inputs and noise is the product probability distribution $p(u)p(\epsilon)$.

We call an ELM {\em noisy} if at least one of the readouts has a nontrivial dependence on the noise variable $\epsilon$, i.e., cannot be expressed as a function of the input $u$ alone. Experimentally, this corresponds to the fact that if we repeat the same experiment with the same input $u$, we will not get exactly the same response $X_i$. The variability is due to the noise. 

When the readouts of a stationary system $\{X_i\}_{i=1}^K$ are linearly independent, the sum of capacities over a complete basis $\{y_l(u)\}$ of $\mathcal{H}$ is equal to $K$ as proven in Prop. \ref{Prop:SumCap}. This is no longer the case for noisy ELMs.

\begin{proposition}
For a noisy ELM, the sum of the capacities over an orthonormal basis $\{ y_l\}$ of the input Hilbert space $ \mathcal{H} $ is strictly less than the total number of readouts
\begin{equation}
  \sum_l   C(y_l,X) < K .
  \label{Eq:SumCapNoise}
\end{equation}
\label{Prop:noise}
\end{proposition}

\begin{proof}
We choose the orthonormal basis of noise functions $\{z_m\}$ such that $z_1=1$ is the constant function, corresponding to the absence of noise.

An orthonormal basis of the space  $\mathcal{H}\otimes\mathcal{H}_\epsilon $ is given by all the products $y_l (u) z_m(\epsilon)$.
Therefore, we have 
\begin{equation}
    \sum_{l,m} C(y_l z_m ,X)=K
\end{equation}
where we have assumed the readouts $X_i$ are linearly independent (otherwise we already have a strict inequality in Eq.
\eqref{Eq:SumCapNoise}).

Because the stationary system is noisy, there is at least one of the readouts, say $X_1$, which has a nontrivial dependence on the noise $\epsilon$.

Since $X_1$ belongs to $\mathcal X$, we may construct an orthonormal basis $\{\tilde X_i\}$ of $\mathcal X$ such that
\[
\tilde X_1 = \frac{X_1}{\|X_1\|}
\]
 and thus has a nontrivial dependence on $\epsilon$. 
 We can expand this basis function as
\begin{eqnarray}
  \tilde{X}_1 &=&\sum_{l,m} c_{lm} y_l z_m\nonumber\\
  &=& \sum_l c_l y_l + \sum_{l,m\neq 1} c_{lm} y_l z_m 
\end{eqnarray}
where we used that $z_1=1$.
The hypothesis that $X_1$ is noisy implies that at least one of the coefficients $c_{lm}, m\neq 1$ is nonzero.
that is, part of the norm of $\tilde X_1$ lies outside the input Hilbert space $\mathcal H$.
Since $\tilde X_1$ is normalised, we have  $\sum_{lm} \vert c_{lm} \vert^2 =1$, and therefore
\begin{equation}
  \sum_l \langle \tilde X_1 , y_l \rangle^2 =  \sum_l c_l^2 <1
\label{Eq:AAA1}
\end{equation}
with strict inequality.
For all other $\tilde X$ we have
\begin{equation}
  \sum_l \langle \tilde X_i , y_l \rangle^2 \leq 1
  \quad i\neq 1
  \label{Eq:AAA2}
\end{equation}
where one would have equality if the corresponding $\tilde X_i$ is not affected by noise.

We now repeat the proof of Prop. \ref{Prop:SumCap}:
\begin{align*}
\sum_{l} C(y_l, X) &= \sum_{l}  C(y_l, \tilde{X}) \\
&= \sum_{l}  \sum_{i=1}^{K} \langle \tilde{X}_i, y_l \rangle^2 \\
& =  
\sum_{i=1}^{K} \left( \sum_{l}  \langle \tilde{X}_i, y_l \rangle^2 \right) \\
&< K
\end{align*}
where we have used Eqs. \eqref{Eq:AAA1} and \eqref{Eq:AAA2}.
\end{proof} 

\subsection{Additive noise}
The preceding section showed that in the presence of noise, the capacities no longer saturate the bound of Prop. \ref{Prop:SumCap}.
As an illustration, we consider here the specific case of additive noise:
\begin{equation}
    \label{eq:X_noise}
    X_i^{\nu}(u ,\epsilon)=X_i (u)+ \nu_i(\epsilon) .
\end{equation}
and denote $X^\nu=\{X_i^{\nu} \}$ the set of noisy readouts.

We have 
\begin{eqnarray}
    E[X_i \nu_j] &=& 0\nonumber\\
    E[\nu_i y] &=& 0 \nonumber\\
    E[\nu_i  \nu_j] &=& \mathcal{V}_{ij}
\end{eqnarray}
where the expectation is taken over both inputs $u$ and noise $\epsilon$, and $\mathcal{V}$ is the covariance matrix of the noise. 
Denoting $G, R$ and $G^\nu , R^\nu$ the Gram matrix and correlation vector in the absence (respectively presence) of noise, we have 
$G^\nu = G + \mathcal{V}$ and $R^\nu = R$. 

\begin{proposition}
In the case of additive noise, if the Gram matrix $G$ in the absence of noise is full rank, if the covariance matrix $\mathcal{V}$ of the noise is full rank, then for any target function $y$ with nonzero capacity, the capacity in the presence of noise $C(y,X^\nu)$ is strictly smaller than in the absence of noise $C(y,X)$.
\end{proposition}

\begin{proof}

 By hypothesis
 $G\succ 0$ and $\mathcal{V}\succ 0 $ are strictly positive. This implies that $
     G^{-1} \succ  (G + \mathcal{V})^{-1}$, see  \cite{Horn_Johnson_1985}.

     Since the capacity $C(y,X^\nu) >0$ is nonzero, the correlation vector $R$ is nonzero.
     
We then have:
\begin{align}
    C(y,X^\nu) &=\frac{R^{\nu T} G^{\nu -1} R^\nu}{\langle y^2 \rangle}\nonumber\\
    &= \frac{R^{T} \left( G + {\mathcal{V}} \right)^{ -1} R}{\langle y^2 \rangle}
    \nonumber\\
    &< \frac{ R^{T} \left( G \right)^{ -1} R}{\langle y^2 \rangle}  \nonumber\\
    &= C(y,X).
\end{align}  
\end{proof}

\section{Finite number of samples}\label{Sec:FiniteN}

In practice, we cannot evaluate the capacities exactly. Rather, we have a finite number of samples $N$. This introduces a systematic positive bias when estimating the capacities. These were previously discussed in \cite{dambre2012information,JMLR:v25:23-0568,SAITO2026132128,infocap_timevar}. We revisit this issue in the simpler case of stationary systems.

When we have a finite number of samples, the definitions in Eqs.
\eqref{Eq:DefMSE} to \eqref{Eq:DefG} are unchanged, but the averages are now the empirical averages over the $N$ samples.
We denote all quantities evaluated for a finite number $N$ of samples with the subscript $N$.

\begin{proposition} 
\label{Prop:CapN}
Given $N$ samples $(u(n), X(n), y(n))$, $n=1\ldots N$, the  capacity $C_N(y)$ of a stationary system to reconstruct a function $y$ is bounded by:
\begin{equation}
0 \leq C_N(y, X) \leq 1 .
\label{Eq:CNbound}
\end{equation}
If the number of samples is less than the number of variables, $N\leq K$, and the $N$ sample vectors $X_i(n)\in\mathbb R^K$ are linearly independent, then the empirical capacity is $1$:
\begin{equation}
C_N(y, X) = 1 \quad (N\leq K) .
\end{equation}
\end{proposition}

\begin{proof} The  argument is the same as in the proof of Proposition \ref{Prop:CapBound}: since the  MSE is positive, we have, using the definition Eq.  \eqref{Eq:Cdef}, that $C_N(y, X) \leq 1$; and since the 
Gram matrix $G_N$ is symmetric and positive semidefinite, we have, using Eq. \eqref{Eq:Cexplicit}, that $0 \leq C_N(y, X)$.

Since the sample vectors are linearly independent, when $N \leq K$, we can solve for $W_i$ the system 
\begin{equation}
\sum_{i=1}^K W_i X_i (n) = y(n) .
\end{equation}
Using this solution, the output $\hat y = \sum_{i=1}^K W_i X_i $ is equal to the target output $y$, and the MSE vanishes. Hence, from Eq.  \eqref{Eq:Cdef}, we have $C_N(y,X)=1$.
\end{proof}

When the number of samples increases, the capacities decrease towards their asymptotic value, as illustrated in  Fig. \ref{Fig:2} (a). We now analyse this asymptotic behavior.

Throughout the remainder of this section, we assume that
$\langle y^2\rangle$ is known analytically and does not need to be estimated from the data. This is true in the important case of estimating the capacities of a complete orthonormal basis such as the one constructed from Legendre polynomials in Section \ref{subsubsec:examp_legendre}.
We also assume throughout that the samples are independently drawn from the underlying distribution $p(u)$.
We then have the following result for the asymptotic behavior of the capacity when the number of samples is large

\begin{proposition} 
\label{Prop:AsymptCN}
Given $N \gg K$ independent samples, the expectation over datasets of the empirical capacity $C_N(y,X)$ for a stationary system to reconstruct a function $y$ can be expanded in a series in $1/N$. Assuming that $\langle y^2\rangle$ is known and does not need to be estimated, the leading terms of this expansion take the form
\begin{equation}
E \bigl[ C_N(y, X) \bigr] = C(y, X) +  \frac{\Delta(y,X)}{N} + \mathcal{O}\left({N^{-3/2}}\right) .
\label{Eq:capacity_N}
\end{equation}
where
\begin{equation}
\label{eq:Delta_N}
 \Delta(y,X) = 
  \frac{  E \left[  \left ( \sum_{i=1}^K  \tilde{X}_i^2 \right)
  \left( P_{\mathcal X}^\perp (y) \right)^2 
  \right] }{   E \left[  y^2 \right]} \geq 0 .
\end{equation}
\end{proposition}

Three remarks about this result before giving the proof. (i) Eq. \eqref{eq:Delta_N} shows that due to statistical fluctuations, the finite sample estimator contains information about the component of $y$ orthogonal to the space ${\mathcal X}$ spanned by the readouts. (ii)
The positivity of $\Delta(y,X)$ does not contradict the orthogonality relation
$ E[\tilde X_i P_{\mathcal X}^\perp(y)]=0$, since the variables $\tilde X_i$ and $P_{\mathcal X}^\perp(y)$ are generally not independent.
 (iii) When the capacity is $1$, $\Delta(y,X)$ vanishes:
\begin{equation}
\Delta(y,X) = 0\text{ when }C(y,X)=1 
\end{equation}
since in this case $P_{\mathcal X}^\perp (y) =0$.

\begin{proof}
We define
\begin{eqnarray}
 \tilde{R}_{Ni} &=&  \frac{1}{N} \sum_{n=1}^N \tilde{X}_i(n)  y(n) \nonumber\\
    \tilde{G}_{Nij} &=&  \frac{1}{N} \sum_{n=1}^N \tilde{X}_i (n)  \tilde{X}_j  (n) 
    \end{eqnarray}
    whose fluctuations around their average converge asymptotically to multivariate Gaussian distributions by the central limit theorem.
    We define $\tilde  \eta_{N} $ as the difference between $\tilde{G}_{N}$ and its expectation:
\begin{equation}
    \tilde{\eta_{N ij}} =\tilde{G}_{Nij}-\delta_{i,j} 
\end{equation}
which implies
\begin{align}
 E \bigl[ \tilde{G}_{ij} \bigr] = E \bigl[ \tilde{G}_{Nij} \bigr] &=\delta_{ij} ,\nonumber\\
E \bigl[ \tilde{\eta}_{ij} \bigr] = E \bigl[ \tilde{\eta}_{Nij} \bigr] &=0 .   
\end{align}
The expectation over datasets of the finite sample capacity $C_N(y,X)$ can be written as:
\begin{eqnarray}
 E \bigl[y^2 \bigr] E\bigl[ C_N(y,X)\bigr] &=& E\bigl[\tilde R_N^T \tilde G_N^{-1} \tilde R_N \bigr]\nonumber\\
 &=&  E \bigl[\tilde R_N^T  \tilde R_N \bigr]
 - E\bigl[\tilde R_N^T  \tilde \eta_N \tilde R_N \bigr] +   E \bigl[\tilde R_N^T  \tilde \eta_N^2 \tilde R_N \bigr] - \mathcal{O}(N^{-3/2})
\end{eqnarray}
where we have used $\tilde G_N^{-1} = \left( \mathbb{I} + \tilde \eta_N \right)^{-1}
=\mathbb{I} - \tilde \eta_N + \tilde  \eta_N^2 - \mathcal{O}(N^{-3/2})$ 
and used the fact that $\tilde \eta_N$ tends to zero for large $N$.

Since the samples are independent, the first term evaluates to
\begin{equation}
    E\bigl[ \tilde{R}_N^T\tilde{R}_N \bigr] = \sum_{i=1}^K E \bigl[ \tilde{X}_i  y \bigr]^2 + \frac{1}{N} \sum_{i=1}^K \Bigl(  E \bigl[ \tilde{X}_i^2 y^2 \bigr] - E \bigl[\tilde{X}_i  y \bigr] ^2 \Bigr) .
\end{equation}
The leading contributions of the second and third terms are of order $1/N$. They are given respectively by
\begin{equation}
    E \bigl[ \tilde{R}^T_N \cdot  \tilde \eta_N \cdot \tilde{R}_N \bigr]= \frac{2}{N} \left(  \sum_{i,j=1}^K  E \bigl[\tilde{X}_i^2 \tilde{X_j} y \bigr] E\bigl[  \tilde{X}_j  y \bigr] -  \sum_{i=1}^K E\bigl[  \tilde{X}_i  y \bigr]^2 \right) +\mathcal{O}(1/N^2) 
\end{equation}
and
\begin{equation}
E \bigl[ \tilde{R}^T_N \cdot {\tilde \eta}_N^2 \cdot \tilde{R}_N \bigr] = \frac{1}{N} \left( \sum_{i,j,k=1}^K  E \bigl[ \tilde{X}_i  y \bigr]   E \bigl[ \tilde{X}_i  \tilde{X}_j^2  \tilde{X_k} \bigr]  E \bigl[ \tilde{X}_k  y \bigr] -  \sum_{i=1}^K E \bigl[ \tilde{X}_i  y \bigr]^2 \right) + \mathcal{O}(1/N^2)
\end{equation}

Putting all together, we have
\begin{eqnarray}
    \label{Eq:CN-intermediate}
E\bigl[ y^2 \bigr] E\bigl[  C_N(y,X) \bigr] &=& \sum_{i=1}^K E \bigl[ \tilde{X}_i \cdot y \bigr]^2  \nonumber\\
& & +\frac{1}{N} \left( 
\sum_{i=1}^K  E \bigl[ \tilde{X}_i^2 y^2 \bigr] - 2
   \sum_{i,j=1}^K  E \bigl[\tilde{X}_i^2  \tilde{X_j}  y \bigr]  E\bigl[  \tilde{X}_j  y \bigr] 
   \right.\nonumber\\
   & & \left.   
   +
\sum_{i,j,k=1}^K  E \bigl[ \tilde{X}_i y \bigr]   E \bigl[ \tilde{X}_i  \tilde{X}_j^2  \tilde{X_k} \bigr] E \bigl[ \tilde{X}_k  y \bigr] \right) 
   \nonumber\\
& & + \mathcal{O}(N^{-3/2}) .
\end{eqnarray}

We can decompose $y$ into a component in the subspace spanned by the $X$ and a component in the orthogonal space:
$
y= P_{\mathcal X}(y) + P^\perp_{\mathcal X}(y)
$.
 We can choose  the orthonormal basis $\tilde X$ such  that $P_{\mathcal X}(y) $ is proportional to $\tilde X_1$:
 \begin{equation}
y= \alpha \tilde X_1 + P^\perp_{X}(y)
\end{equation}
with $\alpha \geq 0$.
Inserting this in Eq. \eqref{Eq:CN-intermediate} yields:

\begin{align}
E\bigl[  C_N(y,X)\bigr] &=  C(y,X)+  \notag\\
&\quad \frac{1}{E\bigl[y^2 \bigr] \cdot N} \Bigl(  \alpha^2 E\Bigl[\sum_{i=1}^K \tilde{X}_i^2 \tilde{X}_1^2\Bigr]+  E \Bigl[ \bigl( \sum_{i=1}^K\tilde{X}_i^2 \bigr) \bigl(P^\perp_{X}(y) \bigr)^2 \Bigr]+ 2 \alpha E\Bigl[ \sum_{i=1}^K \tilde{X}_i^2 \tilde{X}_1 P^\perp_{X}(y) \Bigr] \notag\\
 &\quad -2 \alpha   E \Bigl[\sum_{i=1}^K\tilde{X}_i^2  \tilde{X_1}  y \Bigr]   +  \alpha^2 \sum_{j=1}^K   E \Bigl[ \tilde{X}_j^2  \tilde{X}_1^2 \Bigr] \Bigr) + \mathcal{O}(N^{-3/2}) \\
 &= C(y,X)+ \frac{1}{N} \left( \frac{  E \left[  \left ( \sum_{i=1}^K  \tilde{X}_i^2 \right)
  \left( P_{\mathcal X}^\perp (y) \right)^2 
  \right] }{   E \left[  y^2 \right]}\right)+\mathcal{O}(N^{-3/2})
\end{align}
\end{proof}

Let us now turn to the case where the capacity vanishes, $ C(y, X) =0$. This implies $P_{\mathcal X}(y)=0$ and $P^\perp_{\mathcal X}(y)=y$ which implies the following simpler expression and bound

\begin{proposition} 
\label{Prop:zerothresh}
Given $N \gg K$ samples, for a function $y$ whose  capacity vanishes, $C(y,X)=0$, and assuming that $\langle y^2\rangle$ is known,
we have
\begin{equation}
   C_N(y,X)= \frac{\Delta(y,X)}{N}  +\mathcal{O}(N^{-3/2})
\end{equation}
with 
\begin{equation}
\label{Eq:delta_zero}
\Delta(y,X)= \frac{ E \left[ \left( \sum_{i=1}^K  \tilde{X}_i^2 \right ) y^2 \right] }{ E \bigl[  y^2 \bigr]} 
\leq 
\frac{\sqrt{ E\Bigl[\Bigl(\sum_{i=1}^K  \tilde{X}_i^2\Bigr)^2 \Bigr] \cdot E \bigl[y^4 \bigr]}}{E \bigl[  y^2 \bigr]} .
\end{equation}

\end{proposition}

\begin{proof}
In the case when capacity vanishes, $P_{\mathcal X}^\perp (y)=y$. Therefore, Eq. \eqref{eq:Delta_N} gives the first equality in Eq. \eqref{Eq:delta_zero}.
Applying the Cauchy-Schwarz inequality gives the inequality in Eq. \eqref{Eq:delta_zero}.
\end{proof}

In the case where $y=\mathcal{P}_l$ is an orthonormal  Legendre polynomial of degree $l$, we can explicitly compute the $E \Bigl[y^4 \Bigr]$ term which is given by\cite{adams1878iii}:
\begin{align}
\label{eq:exp_legendre}
 E \Bigl[\mathcal{P}_l^4 \Bigr]= \frac{1}{2}\int_{-1}^{1} \mathcal{P}_l(x)^4\,dx=(2l+1)^2\sum_{L=0}^{l}(4L+1) \begin{pmatrix} l & l & 2L\\0 & 0 & 0\end{pmatrix}^{4}
\end{align}
where $\begin{pmatrix} l & l & 2L\\0 & 0 & 0\end{pmatrix}$ is a Wigner-3j symbol\cite{edmonds1996angular}.

For a product Legendre basis with multi-dimensional inputs 
$u = (u_1, \ldots, u_q)$, the basis functions take the form 
$y_{l_1 l_2 \cdots l_q}(u) = P_{l_1}(u_1) P_{l_2}(u_2) \cdots P_{l_q}(u_q)$.
Since the input components are independent, we have
\begin{equation}
    E[y_l^4]=E\Bigl[y_{l_1 l_2 \cdots l_q}^4 \Bigr] 
    = \prod_{k=1}^{q} E\!\left[P_{l_k}^4\right],
\end{equation}
where each factor $E[P_{l_k}^4]$ is given by Eq. \eqref{eq:exp_legendre}.

\section{Data-efficient estimation of capacities}

We use the results of Section \ref{Sec:FiniteN} to propose methods that enable accurate capacity estimation from limited experimental data. These include algorithms for asymptotic fitting to correct for finite-sample positive bias, removal of false-positive capacities, and low-discrepancy quasi-random sampling to speed up convergence. We validate these methods on a synthetic dataset.

The removal of false-positive capacities is essential if one wishes to obtain a good estimate for the total capacity. This issue was addressed previously in  \cite{dambre2012information,JMLR:v25:23-0568,SAITO2026132128,infocap_timevar}, but  in ways that were either somewhat adhoc, or in the case of \cite{SAITO2026132128} not data efficient. Our approaches to set small capacities to zero are based on the data itself, and are data efficient.

\subsection{Fitting algorithms}
When estimating capacities, we are confronted with the fact that statistical fluctuations will introduce systematic biases to the capacities. Indeed, a capacity $C_N(y,X)$ evaluated on $N$ samples is always positive, see Prop. \ref{Prop:CapN}, and has a systematic positive error of order $1/N$, see Prop \ref{Prop:AsymptCN}. 

 The situation of practical interest is when $K < L < N$, where $K$ is the number of readouts, 
 $L$ the number of orthonormal functions $y_l$, $l=1,\ldots,L$, and   $N$ the number of samples.  We denote the estimated capacities, after the fitting procedure described below, by $\hat C(y_l,X)$. We would like the following conditions to hold:
 \begin{align}
    &0 \leq \hat{C}(y_l, X) \leq 1 ,\label{eq:cond1} \\
    &\hat{C}(y_l, X) \approx C(y_l, X) ,\label{eq:cond2} \\
    &\sum_{l=1}^L \hat{C}(y_l, X) \approx \sum_{l=1}^L C(y_l, X) .\label{eq:sumtildeCN}
\end{align}
The first condition expresses the fact that the estimated capacities $\hat{C}$ should always be positive and bounded by $1$; the second condition that they should be close to the real capacities $C$; and the third condition that the sum of the estimated capacities should be close to the sum of the real capacities. Because $L>K$, condition \eqref{eq:sumtildeCN} does not follow automatically from condition \eqref{eq:cond2}.

We propose two algorithms to address these issues. Both  algorithms exploit the asymptotic behavior $
C_N(y,X)=C(y,X)+\frac{\Delta(y,X)}{N}+O(N^{-3/2})$ proven in Prop. \ref{Prop:AsymptCN}. Using Richardson extrapolation, i.e. computing
$
2C_N(y,X)-C_{N/2}(y,X)
= C(y,X)+O(N^{-3/2})$,
cancels the leading finite-sample bias and provides a more precise estimate of the capacity $C(y,X)$. Because the capacities are asymptotically decreasing with $N$, this
preserves the condition that the capacities should be less than $1$. However, after Richardson extrapolation, we no longer have the guarantee that the capacities are positive. The two algorithms use different methods to address the latter point.

The first procedure is described in Algorithm \ref{alg:capacity_estimation}.
In this algorithm, we first estimate the capacities $C_N(y_l,X)$ and use the basis-dependent upper bound derived from Prop. \ref{Prop:zerothresh} as a threshold, and set all capacities lower than the threshold to zero. Then, for capacities which have not been set to zero, we use Richardson extrapolation\cite{richardson} to estimate the
asymptotic value of the capacity. After this step, a few capacities may still be negative, and these are set to zero. 

\begin{algorithm}[H]
\caption{Estimating capacities: basis function-dependent thresholding followed by Richardson extrapolation.}
\label{alg:capacity_estimation}
\begin{algorithmic}[1]
\Require $N$ samples from the $q$-dimensional inputs $u$, drawn independently from the same distribution $p(u)$, the responses of the $K$ readouts $X_i$ to each input, a set of $L$ orthonormal functions $\{y_l(u), l=1,\ldots,L\}$.
\Ensure Bias-corrected capacity estimates $\{\hat{C}_N(y_l, X)\}_{l=1,\ldots,L}$

\State 
From the data set, compute 
\begin{equation}
    S_N = E_N \left[ \left( \sum_{i=1}^K  \tilde{X}_i^2 \right )^2\right]
    \label{Eq:SN}
\end{equation}

\For{each function $y_l$}

    \State Estimate $C_N(y_l, X)$ using the full dataset of $N$ samples.

\State Compute analytically $E\left[y_l^4\right]$ (given by Eq. \eqref{eq:exp_legendre} for Legendre polynomials)

  \If{${C}_N^{}(y_l, X) \leq \frac{1}{N} \sqrt{S_N E\left[y_l^4\right]}$}
    \label{Alg:T1}
        \begin{equation}
            \hat C_N(y_l,X) \gets  0.
   \label{Eq:Cthreshold1}
        \end{equation}
      \Else ~ Apply Richardson extrapolation:

     \State Partition the dataset into two equal subsets of $N/2$ samples each and calculate capacities on each subset separately to get $C^{(1)}_{N/2}$ and $C^{(2)}_{N/2}$
    \State Estimate $\overline{C_{N/2}}$ by averaging the two resulting estimates
    \begin{equation}
        \overline{C_{N/2}}=\frac{\Bigl(C^{(1)}_{N/2}+C^{(2)}_{N/2} \Bigr)}{2}
    \end{equation}
    \State 
Compute
    \begin{equation}
        \hat{C}_N(y_l, X) = 2 \cdot C_N(y_l, X) - \overline{C_{N/2}} (y_l, X)
        \label{Eq:tildeCN0-B}
    \end{equation}

 \If{$\hat C_N(y_l,X) <0$}
        \begin{equation}
            \hat C_N(y_l,X) \gets  0.
            \end{equation}
\EndIf 

       \EndIf 
     \EndFor  

\State \Return $\{\hat{C}_N(y_l, X)\}$
\end{algorithmic}
\end{algorithm}

We also propose an alternative, simpler Algorithm \ref{alg:capacity_estimation-2}. In this second algorithm, we first perform Richardson extrapolation, and then use Step \ref{Alg:S7} to
ensure that the estimated capacities are always positive.
Capacities which are very small or zero have statistical fluctuations whose magnitude we evaluate with Eq. \eqref{eq:Delta-B}. Setting all capacities less than $-B$ to zero in Eq. \eqref{Eq:Cthreshold2} ensures that the very small capacities do not contribute significantly to the sum $\sum_{l=1}^L \hat C(y_l,X) $, ensuring that Eq. \eqref{eq:sumtildeCN} is satisfied. We note that we could use a more refined method to estimate the threshold 
used in Step \ref{Alg:S7}, for instance, by looking in more detail at the distribution of the negative $\hat C_N(y_l,X)$.

 In practice, for our case, we have found that both algorithms work almost equally well and produce very similar capacity profiles. The benefits of Algorithm \ref{alg:capacity_estimation-2} is that it is a more straightforward method which does not require computing basis-dependent thresholds, while guaranteeing non-negative capacities. On the other hand, we have observed empirically that Algorithm \ref{alg:capacity_estimation-2} gives stronger fluctuations for the sum of the capacities $\sum_{l=1}^L \hat{C}(y_l, X)$.
 For this reason, we use Algorithm \ref{alg:capacity_estimation} in the remainder of this work. 

 Throughout, because the basis functions $y_l$ are analytically normalized, we use the exact value $
E[y_l^2]=1$ when computing the raw capacities $C_N(y_l, X)$ and the analytical thresholds rather 
 than estimating $E[y_l^2]$ empirically.

\begin{algorithm}[H]
\caption{Estimating capacities: Richardson extrapolation followed by thresholding}
\label{alg:capacity_estimation-2}
\begin{algorithmic}[1]
\Require $N$ samples from the $q$-dimensional inputs $u$, drawn independently from the same distribution $p(u)$, the responses of the $K$ readouts $X_i$ to each input, a set of $L$ orthonormal functions $\{y_l(u), l=1,\ldots,L\}$.
\Ensure Bias-corrected capacity estimates $\{\hat{C}_N(y_l, X)\}_{l=1,\ldots,L}$

\For{each function $y_l$}
    \State Estimate $C_N(y_l, X)$ using the full dataset of $N$ samples
    \State Partition the dataset into two equal subsets of $N/2$ samples each and calculate capacities on each subset separately to get $C^{(1)}_{N/2}$ and $C^{(2)}_{N/2}$
    \State Estimate $\overline{C_{N/2}}$ by averaging the two resulting estimates
    \begin{equation}
       \overline{ C_{N/2}}=\frac{\Bigl(C^{(1)}_{N/2}+C^{(2)}_{N/2} \Bigr)}{2}
    \end{equation}
    \State \label{Alg:S5}
    Apply Richardson extrapolation using the asymptotic behavior
           of Proposition~\ref{Prop:AsymptCN}:
    \begin{equation}
        \hat{C}_N^{}(y_l, X) = 2 \cdot C_N(y_l, X) - \overline{C_{N/2}}(y_l, X)
        \label{Eq:tildeCN0}
    \end{equation}
\EndFor

\State Compute the minimum over all extrapolated estimates:
\label{Alg:S7}
\begin{equation}
    B = \min_l\, \hat{C}_N^{}(y_l, X)
    \label{eq:Delta-B}
\end{equation}
\For{each function $y_l$}
    \If{$\hat{C}_N^{}(y_l, X) < -B$}
    
        \begin{equation}
            \hat C_N(y_l,X) \gets  0.
   \label{Eq:Cthreshold2}
        \end{equation}
    \EndIf
\EndFor

\State \Return $\{\hat{C}_N(y_l, X)\}$
\end{algorithmic}
\end{algorithm}

\subsection{Low-discrepancy quasi-random sampling}

To further improve data efficiency, we propose the use of quasi-random sampling to generate the input sequence $u$. Sobol sampling~\cite{SOBOL196786} is a method for generating low-discrepancy, quasi-random, uniformly distributed samples, which is widely used to improve convergence rates in Monte Carlo simulations. Sobol sampling is especially useful in the case of higher-dimensional inputs, since pseudo-random sampling can often form clusters and fail to properly represent the whole space with a limited number of points.

For Sobol sampling, it is important to use ordered subsets (i.e., the first and second halves) when dividing the samples for Richardson extrapolation in Algorithms \ref{alg:capacity_estimation} and  \ref{alg:capacity_estimation-2}, as selecting an unordered subset negates the low-discrepancy advantages of Sobol sampling.

\subsection{Validation}
\label{subsec-val}

To validate Algorithm  \ref{alg:capacity_estimation}, we tested it on a synthetic dataset for which we know exactly the capacities for each basis function.
The data set is 
generated as follows.
We used 5-dimensional inputs, generated 8192 samples, and considered all Legendre product basis functions of total degree less than or equal to 8. This corresponds to $1287$ basis functions. We then randomly selected $200$ of these functions, and finally generated $71$ synthetic readouts as random linear combinations of these $200$ selected basis states. The number of readouts $K=71$, input dimensions $q=5$, number of samples $N=8192$, and the maximum degree were chosen to match the analysis of the experimental system presented later in this paper.
We repeated this procedure 1000 times, each time choosing a different set of 200 basis states and synthetic readouts, and computed the mean of the error between the actual capacity and the fitted capacity   $\hat C_N(y_l,X) - C(y_l,X)$  for each basis $y_l$.

To assess the effectiveness of Sobol sampling in estimating capacities, we repeated the validation procedure using both a pseudo-random uniform distribution and a Sobol distribution with the same number of samples.

The results, reported in Fig \ref{Fig:2} (b) and (c), indicate that the validation method performs well in thresholding zero capacities and accurately capturing the asymptotic capacity. Furthermore, Sobol sampling is observed to have smaller errors than pseudo-random uniform sampling. For this reason, Sobol sampling is used throughout the paper for capacity estimation.

\begin{figure}[H]
    \centering
    \includegraphics[width=1\textwidth]{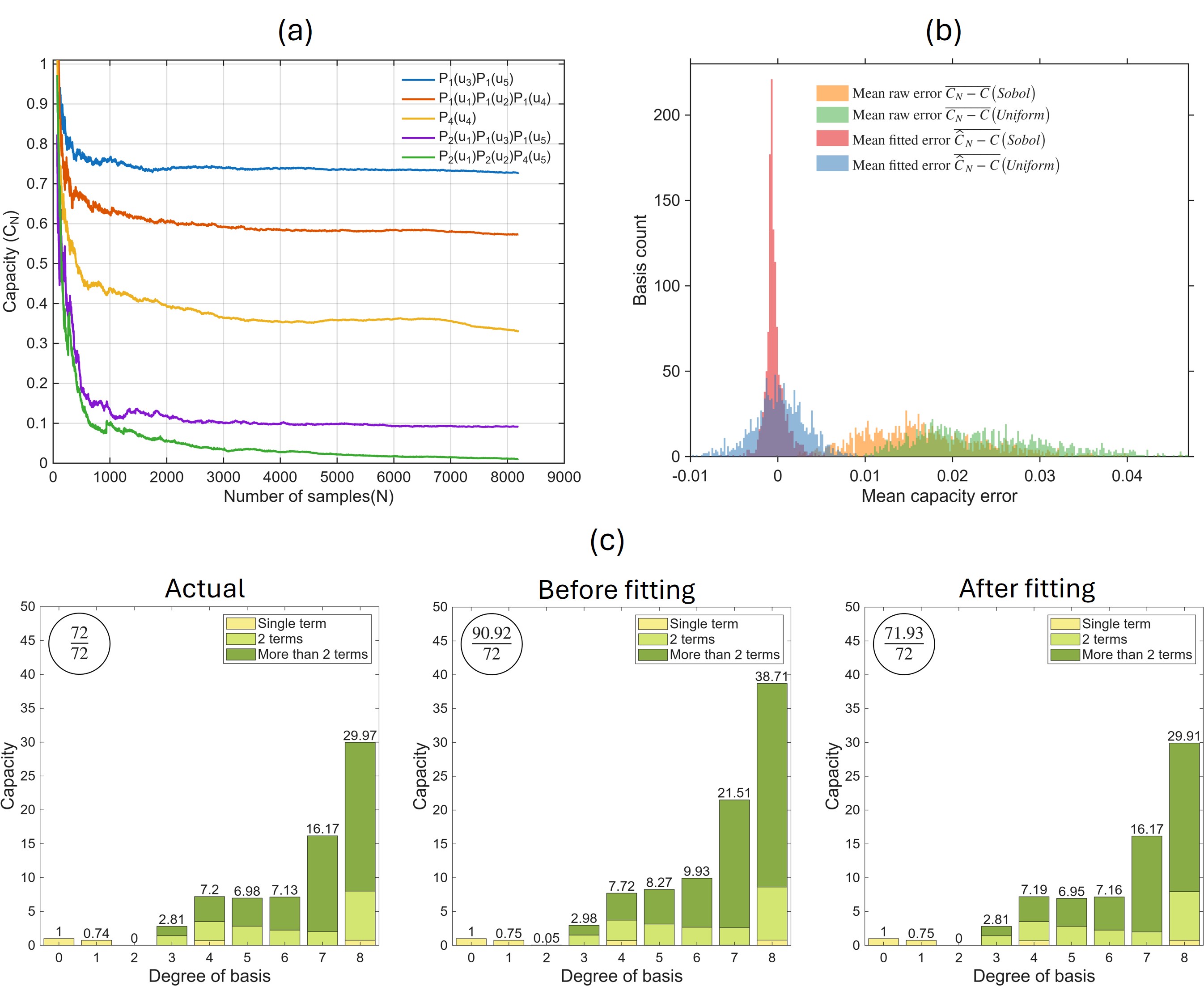}
    \caption{
    (a) Raw capacities $C_N(y_l, X)$ as a function of the number of samples $N$ for selected basis functions. The curves all start at $1$ for $N=K=72$, and then decrease for larger $N$, illustrating the asymptotic behaviour proven in Prop. \ref{Prop:AsymptCN}. The data for this panel is obtained using the experimental system described in Section \ref{Sec:Exp-Sys} using average power=-1.9 dBm and fiber length =40m.
    (b,c) Results for the synthetic dataset described in Section \ref{subsec-val}. 
  (b)  Histogram of the errors made when estimating capacities. Horizontal axis is the error made when estimating capacities with bins of width $2 \times 10^{-4}$. Vertical axis is the number of basis functions (products of Legendre polynomials) with error falling within one bin.
  Yellow, Green: Error $ C_N(y_l,X) - C(y,X)$ between the true capacity and the raw capacity with Sobol and pseudo-random uniform distributions, respectively. Red, Blue: Error  $ \hat C_N(y_l,X) - C(y_l,X)$ between true capacity $C(y,X)$ and the estimated capacity $\hat C_N(y_l,X)$  obtained after applying Algorithm \ref{alg:capacity_estimation} with Sobol and pseudo random uniform distributions, respectively. We see that  Algorithm \ref{alg:capacity_estimation} removes a systematic positive bias and reduces the dispersion of errors, and that Sobol sampling reduces the variance of the errors.
  (c) Capacity bar-plots for the true capacities (Actual, left panel), the raw values $C_N(y_l,X)$ (Before Fitting, middle panel), and the estimates  $\tilde C_N(y_l,X)$ obtained after applying Algorithm \ref{alg:capacity_estimation} (After Fitting, right panel). Applying Algorithm \ref{alg:capacity_estimation} makes the estimated capacities, and in particular the total capacity (inset inside the circle), much closer to the ground truth. Sobol sampling is used in panels (a) and (c).}
  \label{Fig:2}
\end{figure}

\section{Experimental setup}
\label{Sec:Exp-Sys}

\subsection{Physical system}\label{subsec:PS}

\begin{figure}[H] 
    \centering
    \includegraphics[width=1.0\textwidth]{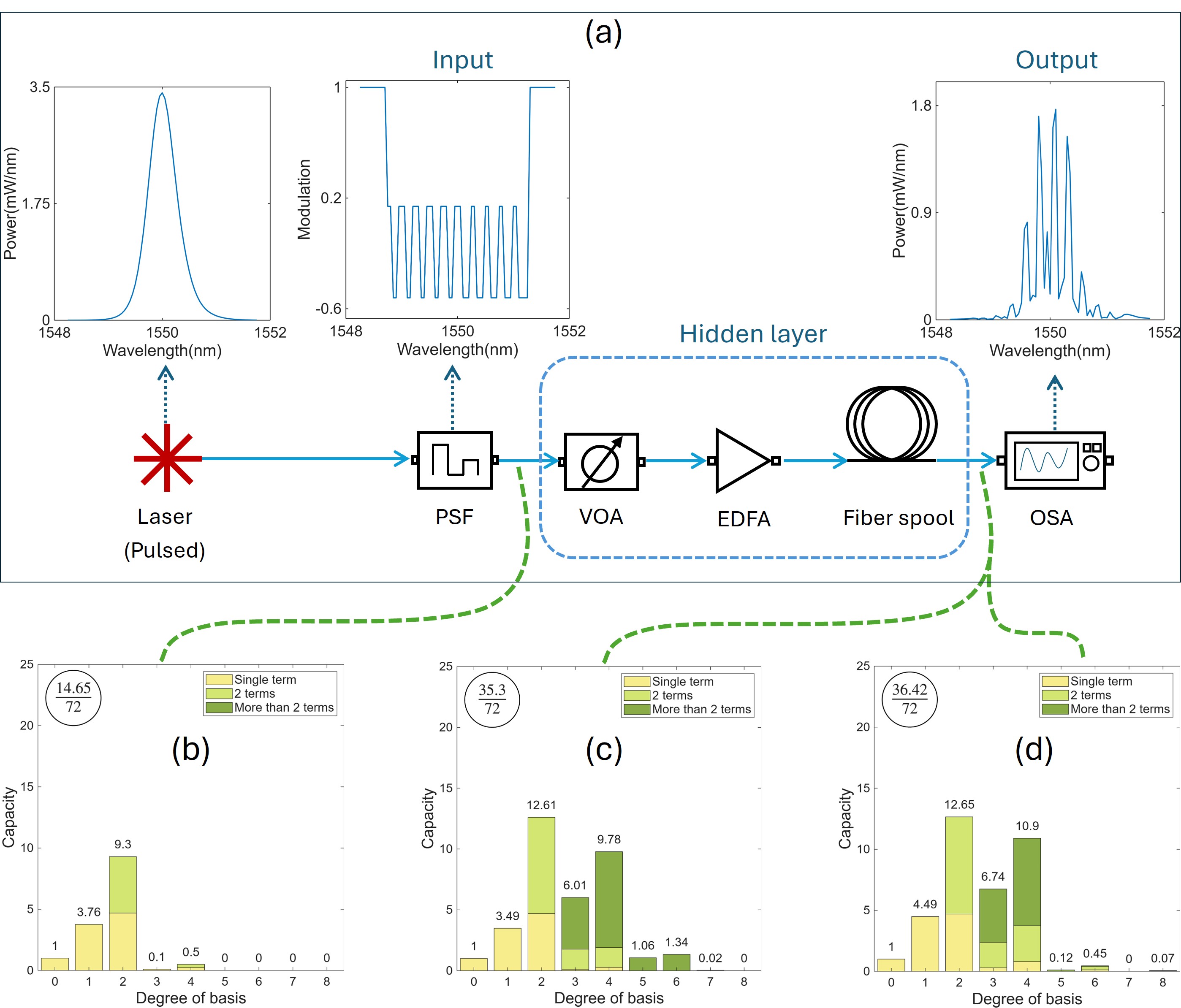}
    \caption{(a) Schematic of the Experiment. Programmable Spectral Filter (PSF), Variable Optical Attenuator (VOA), Erbium Doped Amplifier (EDFA), Optical Spectral Analyzer (OSA).  The insets depict the spectrum of the laser, the transmission of the PSF when 2 inputs are encoded, and the corresponding output spectrum. Information Processing Capacities depicted as bar plots: (b) Simulation of IPC before propagation; (c) Simulation of IPC  after propagation through the fiber; (d) Experimental measurements of IPC  after propagation. All results are for average output power of $6.6$ dBm and fiber length of $5$m.}
    \label{Fig:exp_system}
\end{figure}

The experimental system, depicted in Fig. \ref{Fig:exp_system}, follows previous works  \cite{nonlin_schro,supercontinuum}.
The source is a laser (Pritel FFL) that produces pulses with a full width at half maximum (FWHM) of 4.2 ps, peak power 61 W, spectral width of 0.6 nm (FWHM) centered around 1550 nm, and repetition rate 10 MHz.

The input encoding stage is implemented using a Programmable Spectral Filter (PSF, Finisar Waveshaper 4000S). 
The PSF is programmed to encode the inputs in a spectral span of 2.5nm divided into 20 discrete bins. In the case of 2 inputs, the spectral bins alternately encode one and then the other input. In the case of 5 inputs, the 5 inputs are encoded in sequential spectral bins, and the sequence is repeated 4 times. The inputs belong to the interval $u\in[-1,+1]$. We instruct the PSF to attenuate the power of each spectral bin by a factor of $|u|$ 
such that $|u|=0$ corresponds to the maximum attenuation and $|u|=1$ corresponds to the minimum attenuation. We also apply a phase of $0$ or $\pi$ according to the sign of $u$. This corresponds to multiplying the spectral field by $\mathrm{sign}(u) \sqrt{|u|}$

The pulse then passes through a Variable Optical Attenuator (VOA, Thorlabs DV1550AA), and then through an Erbium Doped Fiber Amplifier (EDFA, Pritel PMFA 15). The average output power can be adjusted in the range -9.1 dBm to 7.1 dBm. The amplified pulse then passes through a fiber spool. Finally, the spectrum is measured using an Optical Spectrum Analyzer (Yokogawa AQ6370D) with a resolution of 0.05nm.
A spectral span of 3.5nm is used, resulting in  71 spectral bins, which are used as outputs.

The experiment involves two controllable parameters: (1) the laser power, adjusted via the VOA while the EDFA pump current is held fixed; and (2) the fiber length, set by switching between 5m and 40m fiber spools.

\subsection{Model of the physical system}

A simple model of the experiment is as follows.
The pulse envelope (ignoring the carrier wave) at the output of the laser is taken to be a sech function:
\begin{align}
    A(t)=\sqrt{P_{peak}}\operatorname{sech}(t/\tau) ,
\end{align}
where $P_{peak}$ is the peak power at the output of the laser and $\tau$ is related to the pulse width (FWHM) by $\tau_{\text{FWHM}}=1.763\tau$.
The corresponding envelope in the spectral domain is:
\begin{equation}
    \tilde{A}(\omega)= \pi \tau  \sqrt{P_{peak}} \operatorname{sech} \left( \frac{\pi \tau \omega}{2} \right)\ .
\end{equation}
The input is encoded using a function $u(\omega)$ as described in Section \ref{subsec:PS}. To account for the finite instrument response of the PSF, the encoding mask is first convolved with the instrument filter \(h_\mathrm{IF}(\omega)\) of the PSF to yield the effective mask on the spectrum as: 
\begin{equation}
    v(\omega)
    =
    \Bigl(
    \mathrm{sign}(u(\omega))
    \sqrt{|u(\omega)|}
    \Bigr)
    * h_{\mathrm{IF}}(\omega) .
    \label{eq:exp_IF}
\end{equation}

The instrument filter \(h_\mathrm{IF}(\omega)\) of the PSF induces some mixing between inputs.

The amplitude in the spectral domain, after encoding the input and passing through the EDFA, is thus
\begin{equation}
    \tilde{A}(\omega)= \pi \tau   v(\omega) \sqrt{ P_{peak}} \operatorname{sech} \left( \frac{\pi \tau \omega}{2} \right) ,
\end{equation}
with $P_{peak}$ the peak power after amplification.

The pulse envelope $A(t,z)$ then evolves while propagating through the fiber according to the Non Linear Schrödinger Equation:
\begin{equation}
    \frac{\partial A}{\partial z} = -\frac{ \alpha}{2} A +i \frac{\beta_{2}}{2} \frac{\partial^{2} A}{\partial t^{2}}+i\gamma|A|^{2} A 
\end{equation}
where $\beta_2= -23~\mathrm{ps^2/km} $ is the Group velocity dispersion parameter and $\gamma = 1.2~\mathrm{W^{-1}km^{-1}} $ is the nonlinear parameter. The pulse envelope at the end of the fiber is $A(t,L)$.

At the end of the fiber, the power spectrum $\vert \tilde A(\omega, L) \vert^2$ is measured using the OSA. Note that this measures the square of the optical field, thereby automatically generating a quadratic dependence in the field amplitude.

 The strength of the optical nonlinearity can be quantified by the nonlinear phase shift defined as $\phi_{NL}=\gamma P_{peak}L$, where L is the propagation length. Figure ~\ref{Fig:exp_system}(b) depicts the capacities of the system in the absence of propagation. The additional nonzero capacities appearing in  Figure ~\ref{Fig:exp_system}(c) and Figure~\ref{Fig:exp_system}(d) can be attributed to propagation through the fiber.

\subsection{Computational methods}
The uniformly distributed random inputs necessary for capacity estimation are generated via Sobol (quasi-random) sampling over $[-1,1]$. This improves space-filling uniformity compared to standard uniform random draws and accelerates convergence. The number of samples drawn is 4096 and 8192 for 2-dimensional and 5-dimensional inputs, respectively. 

We use the Legendre product basis from Section~\ref{subsubsec:examp_legendre} for all calculations. We retain all basis terms up to a total degree of 14 and 8 for 2-dimensional and 5-dimensional inputs, respectively. This corresponds to 120 basis terms for 2-dimensional, and 1287 terms for 5-dimensional inputs.

In Figure \ref{Fig:exp_system} (b) and (c), the experimental system is simulated using the split-step Fourier method. The instrument filter of the PSF is simulated as a flattop shape. In order to simulate phase instabilities in the experimental system, phase noise $\delta\phi \sim \mathcal{N}(0, \sigma_\phi^2)$ with 
$\sigma_\phi = 0.15 \times 10^{-2} \times 2\pi$ rad is added independently 
to each spectral bin of the encoding mask before convolution with the PSF 
instrument filter (Eq.~\eqref{eq:exp_IF}). The length of the fiber inside the EDFA is 2.4m. The effective length of the EDFA spans between 0.4-0.88m for the powers used in this work. Due to this comparatively small value, EDFA is ignored in simulations and when calculating nonlinear phases.

 The output spectrum provided by the OSA is oversampled. We linearly interpolate the raw spectrum over a width of 0.05nm, equal to the spectral resolution of the OSA.  Given the total spectral width of 3.5nm, this gives us 71 spectral bins as readouts.

Before computing the capacities for a given set of readouts, we append an additional constant readout (all ones) to allow the system to represent target functions that differ only by an additive bias. This increases the maximum total capacity by at most 1.

For the machine learning analysis, we use two benchmark classification tasks. The first one is the Two Spirals task \cite{LangWitbrock1988}. It's a binary classification task in which the goal is to predict which of two interleaving spirals a point belongs to based on its \((x,y)\) coordinates. We use 2000 data points for this task. The \((x,y)\) coordinates are generated in the range $[-1,1]$, which can be directly encoded to the spectrum. The second task uses a simplified MNIST handwritten digit dataset\cite{mnist}. Only a random subset of 5000 samples from the 60000 available training samples is used.  To properly compare with the 5-dimensional information-capacity setting, Principal Component Analysis (PCA)\cite{hotelling1933pca} is used to reduce the input features from 784 to 5. The values are then linearly normalized to be in the range $[-1,1]$. The accuracy of both tasks is estimated using 5-fold cross-validation.

\section{Results}

\subsection{Capacity and experimental parameters}

We first compare the system's total capacity across different combinations of input power and fiber length (the two adjustable parameters). To probe the effect of input dimensionality, we evaluate two cases: \(2\)-dimensional and \(5\)-dimensional inputs.

\begin{figure}[H]
    \centering
    \includegraphics[width=1.0\textwidth]{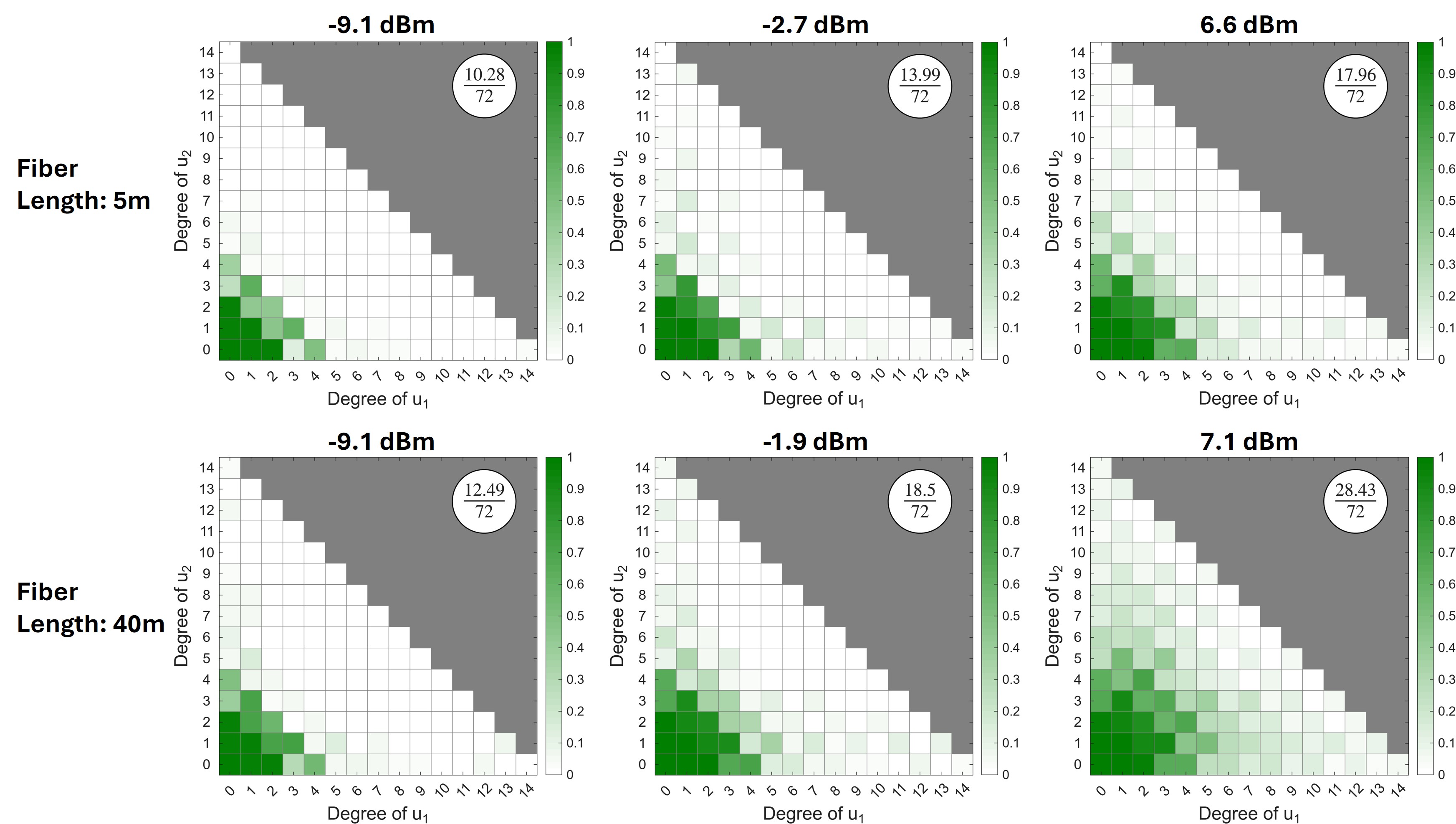}
    \caption{Capacity matrices for 2-dimensional inputs compared at different laser powers (the measured quantity is the average output power) and fiber lengths. The fraction inside the circle corresponds to $ \frac{\textit{Total Capacity}}{\textit{Maximum Capacity}}$.
    }
    \label{Fig:2dResults}
\end{figure}

Figure \ref{Fig:2dResults} shows the results for the
2-dimensional inputs. At low powers, the capacity matrices contain only low-degree terms. In particular, the second-degree terms are always maximum even at small powers. These terms are due to
 the quadratic detection nonlinearity and the convolution with the instrument filter of the PSF. 
They are present even without propagating through the fiber, see Fig. \ref{Fig:exp_system}(b). 
  As the power increases, higher-order terms emerge and grow in magnitude, due to the fiber nonlinearity.

\begin{figure}[H]
    \centering
    \includegraphics[width=1.0\textwidth]{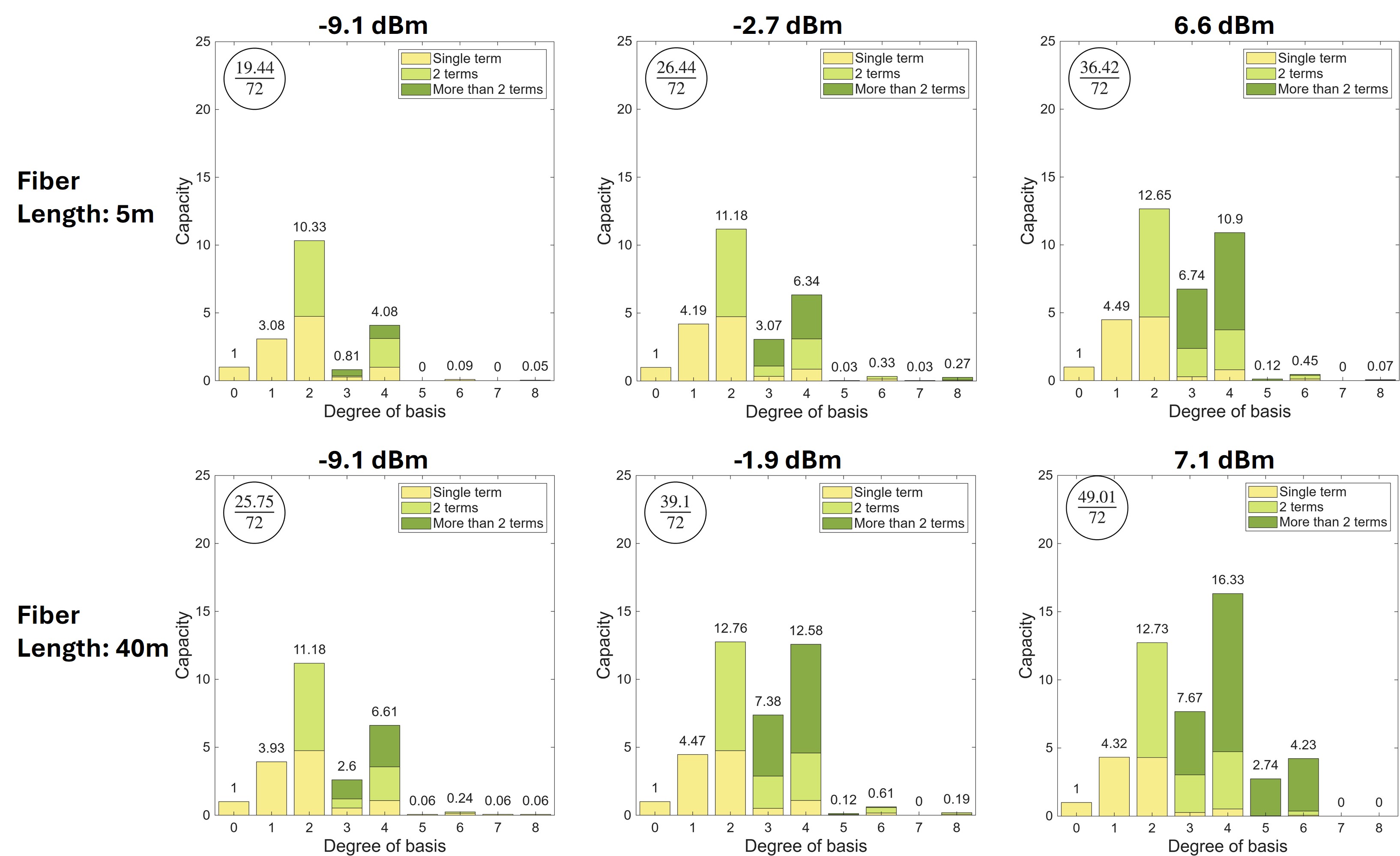}
    \caption{Capacity bar-plots for 5-dimensional inputs compared at different laser powers (the measured quantity is the average output power)  and fiber lengths. The fraction inside the circle corresponds to $ \frac{\textit{Total Capacity}}{\textit{Maximum Capacity}}$}
     \label{Fig:5dResults}
\end{figure}

Figure \ref{Fig:5dResults} shows the results for the
5-dimensional inputs. At low power, capacity is again concentrated in degree 2 terms, while at higher power and longer fiber length, the distribution broadens toward higher total degrees, reflecting a richer set of nonlinear interactions among input components.

In our experiment, the dispersion has only a small effect, and the pulse distortion is dominated by the Kerr nonlinearity whose strength can be measured by the nonlinear phase $\Phi_{NL}$. This is confirmed by
 Fig. \ref{Fig:TotalCapPhiNL}  in which we plot the total capacity as a function of $\phi_{NL}$. We see that the experimental points for 2D and 5D inputs follow different curves, but that the curves for 5m and 40m fiber lengths overlap, suggesting a universal behaviour that depends only on $\Phi_{NL}$.
Across both input dimensionalities, increasing either power or fiber length increases the effective nonlinear phase and therefore shifts capacity from predominantly low-order components toward a more diverse, higher-order profile.

\begin{figure}[H]
    \centering
    \includegraphics[width=0.65\textwidth]{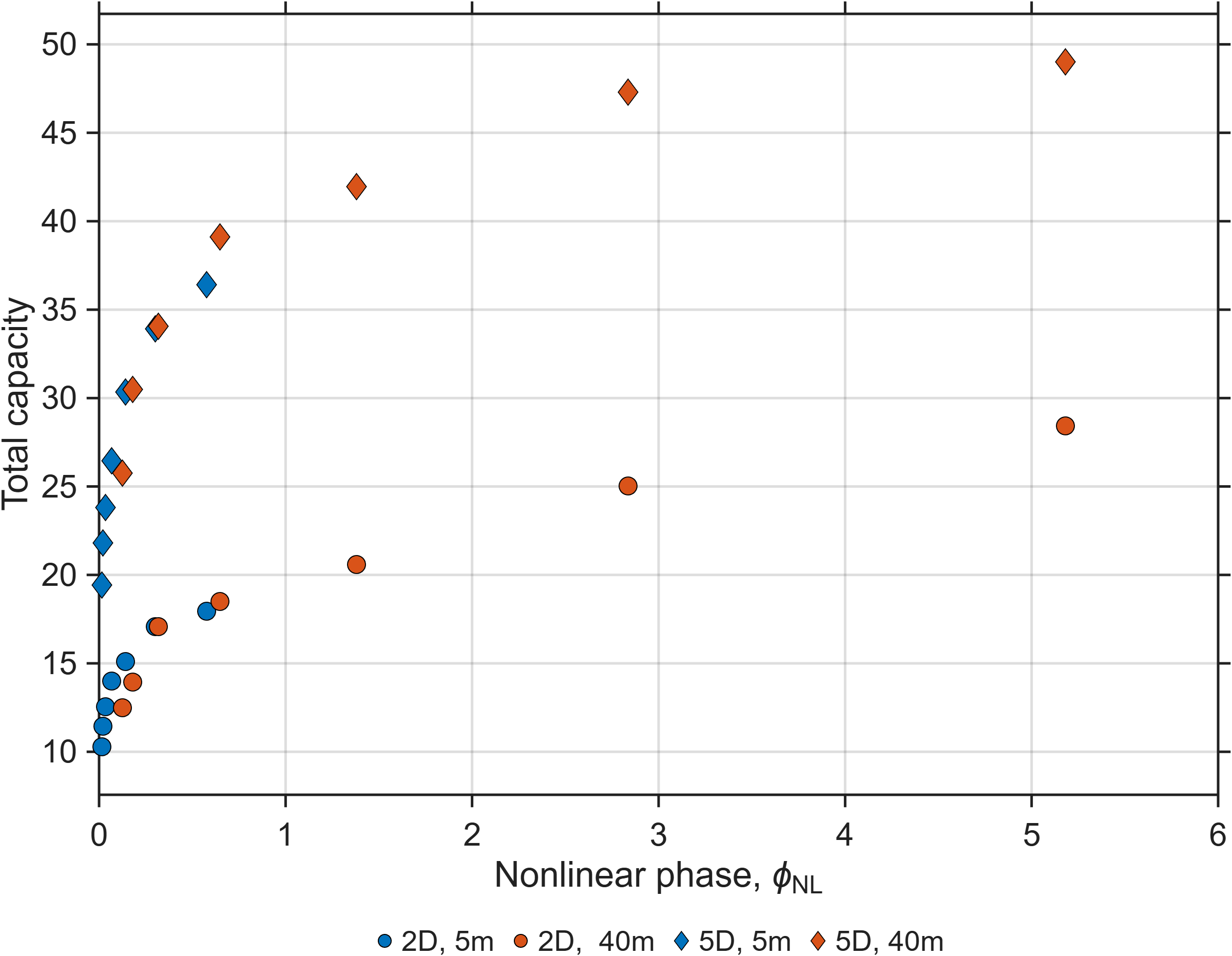}
    \caption{Total capacities as a function of nonlinear phase $\phi_{NL}$, for 2-dimensional  and 5-dimensional inputs. The peak powers used are in the range 2.6 W - 107 W, and the fiber propagation lengths are 5m and 40m}
    \label{Fig:TotalCapPhiNL}
\end{figure}

\subsection{Capacity and machine learning tasks}

\begin{figure}[H]
    \centering
    \includegraphics[width=1.0\textwidth]{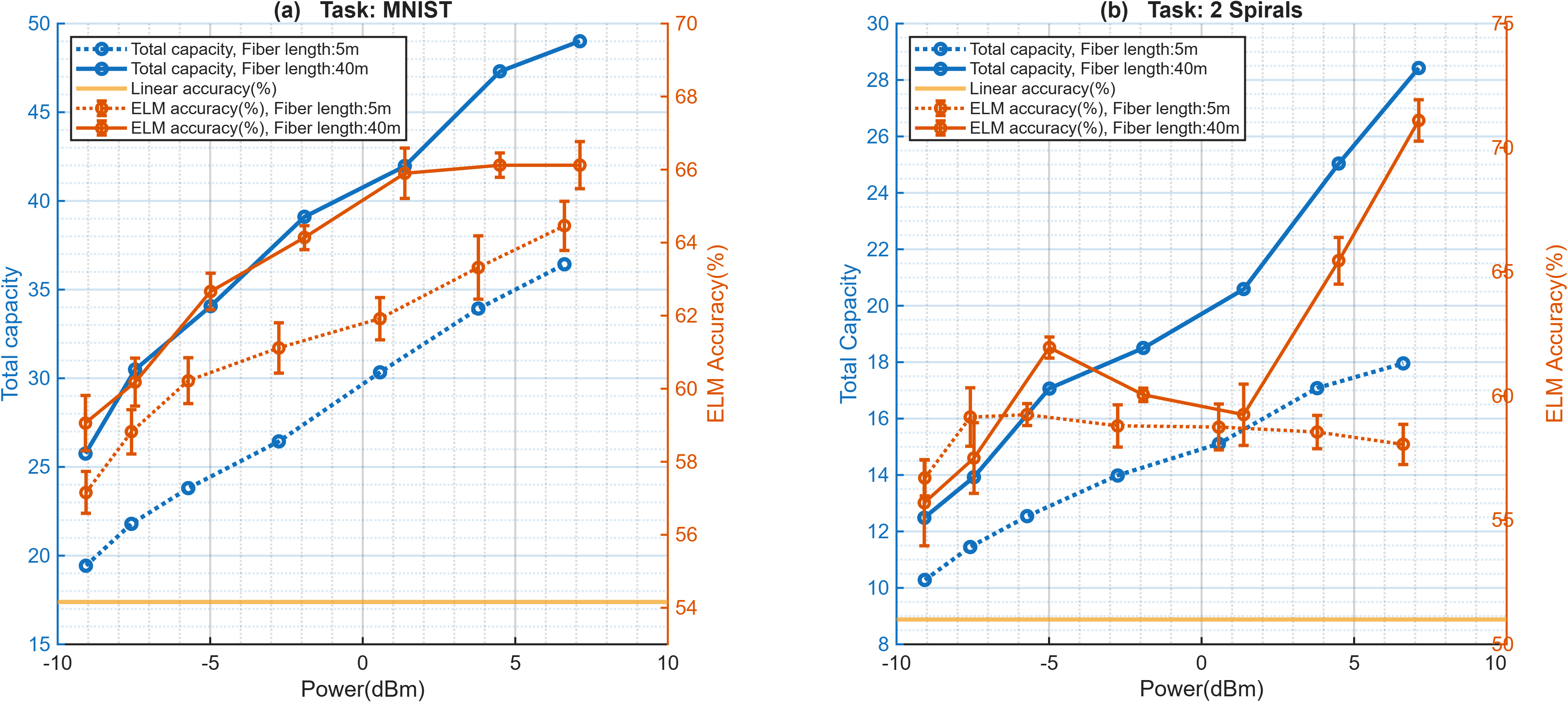}
    \caption{Capacities at different laser powers and fiber lengths compared with accuracies on machine learning tasks. (a): Total capacity for 5-dimensional inputs compared with the Photonic ELM accuracy of the dimensionality-reduced MNIST task.(b): Total capacity for 2-dimensional inputs compared with the Photonic ELM accuracy of the 2 Spirals task. The accuracy obtained from a linear model is also plotted for reference. The error bars on the accuracy are standard errors from the 5-fold cross-validation. 
   The total capacities are the same as those reported in  
Fig \ref{Fig:TotalCapPhiNL}.}
\end{figure}

We compare the system's total capacity with the accuracies achieved on machine learning tasks to assess how capacity is related to the performance of a physical computing system on real-world learning problems. We consider two benchmark tasks matched to our input dimensionalities.

For the \(2\)-dimensional input setting, we use the Two Spirals task \cite{LangWitbrock1988}, a binary classification task in which the goal is to predict which of two interleaving spirals a point belongs to based on its \((x,y)\) coordinates. 
For the 5~m fiber length, the accuracy remains in the range of  \(\sim 50-60\%\) for all input powers.
For the
 40~m fiber length, the accuracy exhibits a clear improvement above about 1~dBm of input power, 
 reaching \(\sim 70\%\) accuracy at around 7~dBm. Around 1 dBm is where the capacity terms with degree $\geq 6$ start becoming significant. So, it’s reasonable to assume this task benefits from (or even requires) these higher-order functional forms, and that increasing the power enables the system to compute them, which in turn improves the accuracy.

For the 5-dimensional input setting, we use the 5-feature MNIST digit-classification task. Here, we see an even clearer correspondence between information capacities and machine learning accuracies. As the input power increases, both the total capacity and the accuracy of the task increase. The 5 m and 40 m capacity curves increase with almost the same slope, but the 5 m curve consistently lies below the 40 m curve across the power range, indicating that the longer fiber length yields higher capacities overall. The corresponding accuracies follow the same trend, indicating that machine learning performance is strongly correlated with capacities.

It must be noted that the total capacity does not always correlate well with task performance. Indeed, not all capacity terms are necessary for carrying out a given machine learning task. Therefore, an increase or decrease in the nonessential capacities will not change the performance on the task. In the context of reservoir computers, the correlation between task performance and capacities has been studied previously in \cite{grigoryeva2015optimal,goldmann2020deep,infocap_task_specific}.

\subsection{Comparison with factor analysis}

In this section, we compare capacity with another metric for benchmarking physical computing systems: factor dimensionality. Principal Component Analysis (PCA) and Factor Analysis are standard techniques for estimating the effective dimensionality or the number of significant independent factors present in a dataset. These methods typically rely on singular value decomposition, after which only a subset of factors is retained based on the magnitude of the corresponding singular values. In \cite{malinowski}, the authors proposed a method for estimating the effective dimensionality by removing noise-dominated principal components. They introduced an indicator function whose minimum identifies the estimated number of significant factors. This metric was recently used in \cite{Skalli:22} to estimate the effective dimensionality of a photonic computing system.

We compute the indicator function using the same dataset employed for estimating the capacities. However, the analysis in \cite{malinowski} assumes that the noise level is uniform across all readouts, whereas in our experiments, the noise varies across the spectrum. To equalize the noise level across readouts, we normalize each readout by its corresponding standard deviation. The noise standard deviation for each spectral bin is estimated by operating the system with a constant input (spectral mask $u(\omega)=1$) for \(100\) trials and computing the standard deviation across those measurements. Let \(Z\) denote the noise-normalized data matrix, and let \(\lambda_i\) denote the eigenvalues of \(Z^T Z\), ordered in decreasing magnitude such that \(\lambda_1 > \lambda_2 > \cdots\). The indicator function \(\text{IND}\), expressed as a function of the number of factors \(\kappa\), is then given by:

\begin{equation}
    \text{IND}(\kappa)=\frac{1}{(K-\kappa)^2} \Biggl[  \frac{\sum_{j=\kappa+1}^K \lambda_j}{N(K-\kappa)}\Biggr]^{1/2}
    \label{Eq:INDfct}
\end{equation}
where, $K$ is the number of readouts and $N$ is the number of samples.
The effective dimensionality of the system is taken to be the minimum of \(\text{IND}(\kappa)\), if such a minimum exists.
An example of experimentally measured indicator function is given in Fig. \ref{Fig:IF}(a).

Comparing the total capacity and factor dimensionality under identical experimental conditions, see Fig. \ref{Fig:IF}(b), we observe a clear correlation between the two metrics. This 
confirms that the 
total capacity can serve as a reliable measure for the effective dimensionality of the system.

Computing the IPC is more data-intensive and computationally demanding than computing the factor dimensionality using the indicator function Eq. \eqref{Eq:INDfct}. 
However, IPC offers two important advantages over the factor dimensionality. First, the factor dimensionality only provides the number of principal factors, while the full capacity profile over a complete basis provides deeper insights into the information processing characteristics of the system. Second, an accurate estimate of the number of factors requires some prior knowledge of the noise in the system, whereas the computation of IPC does not rely on such an assumption.

\begin{figure}[H]
    \centering
    \includegraphics[width=1.0\textwidth]{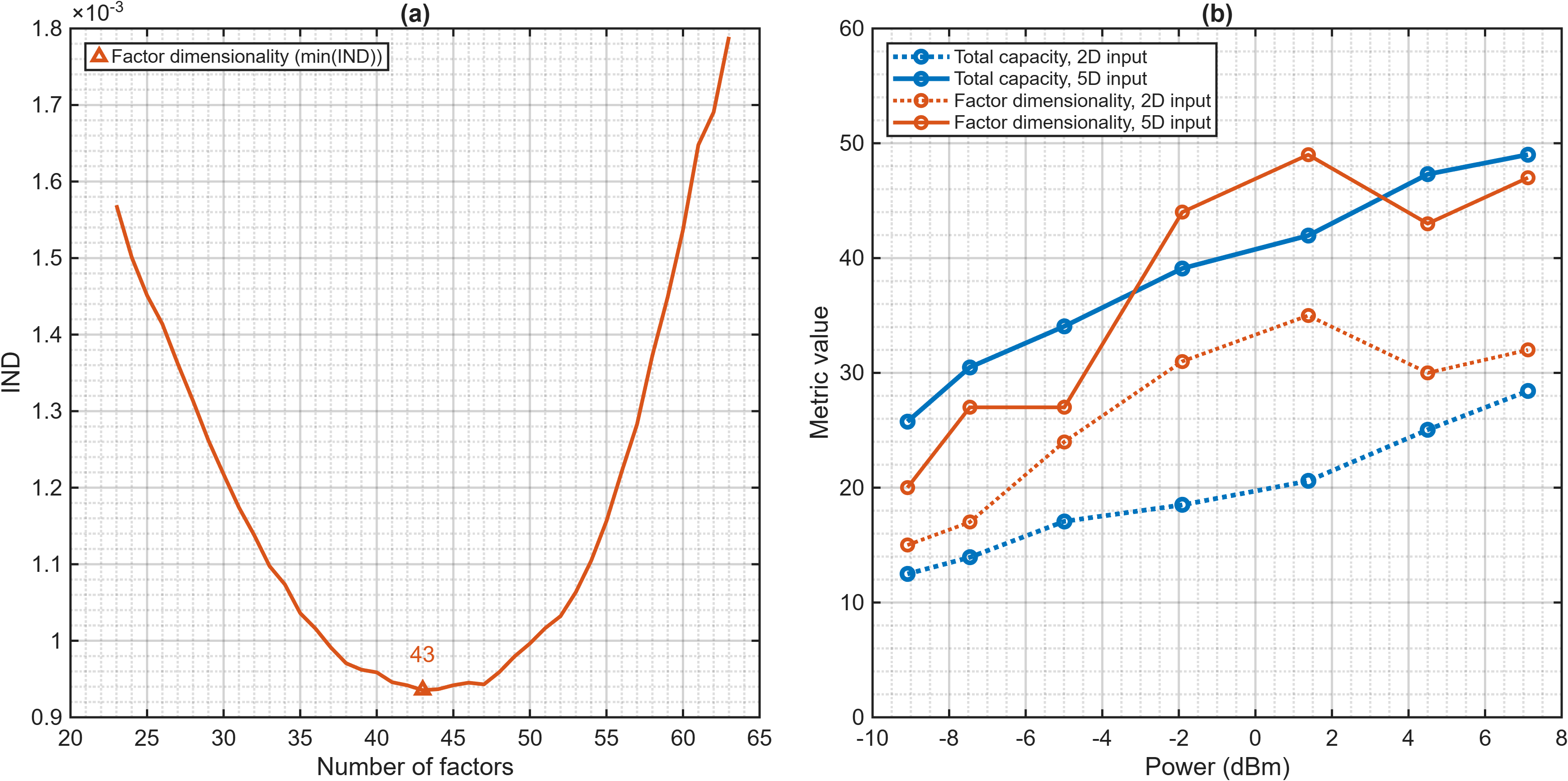}
    \caption{ Estimating effective dimensionality using factor analysis and IPC. (a) Experimentally measured indicator function and its minimum for 5D inputs when average power=4.5 dBm. (b) Total Capacity and Factor Dimensionality compared for both 2D and 5D input at different output average powers.  The length of the fiber spool is  40m.   The total capacities are the same as those reported in  
Fig \ref{Fig:TotalCapPhiNL}. }
    \label{Fig:IF}
\end{figure}

\section{Conclusion}
In the present work we established a theoretical basis for applying the Information Processing Capacity (IPC) to stationary systems. This extends the applicability of IPC from the time-dependent systems, corresponding to reservoir computers, for which it was originally introduced \cite{jaeger2002short,dambre2012information}, to the stationary case, corresponding to Extreme Learning Machines. 

We showed that noise inevitably reduces the IPC.
We introduced data-efficient estimation procedures that correct finite-sample overestimation, enabling more reliable capacity measurements under experimental constraints. Using a photonic ELM based on spectral encoding, fiber propagation, and intensity detection, we showed that controllable physical parameters -input power and fiber length- systematically reshape the capacity distribution, with increased nonlinearity promoting higher-degree components. As an illustrative application of the IPC, we used our experimental system on the  Two Spirals and PCA-reduced MNIST tasks. We found that task performance correlates to some extent with total capacity, demonstrating its predictive potential. We also used the total capacity as a measure of the effective dimensionality of our experimental system, and showed that it correlates strongly with an alternative measure of effective dimensionality based on factor analysis. These results demonstrate the potential of  IPC as a task-independent metric for evaluating and comparing the computational capabilities of stationary physical systems.

\section*{Acknowledgments}
SM would like to thank Daniel Brunner, André Rhöm, and Anas Skalli for insightful discussions. 
This research was supported by the F.R.S.-FNRS CDR J.0143.24
and by the FWO and F.R.S.-FNRS Excellence of Science (EOS) program grant 40007536.

\section*{Data availability}
The data that support the findings of this article are available from the corresponding author upon reasonable request. The code used in this study is openly available at \cite{rahul2026infoprocap}
\bibliographystyle{ieeetr}
\bibliography{references}

\end{document}